\documentclass[lettersize,journal]{IEEEtran}
\usepackage{amsfonts}
\usepackage{amsmath}  
\usepackage{amssymb}  
\usepackage{amsthm} 
\usepackage{algorithmic}
\usepackage{algorithm}
\usepackage{array}
\usepackage[caption=false,font=normalsize,labelfont=sf,textfont=sf]{subfig}
\usepackage{textcomp}
\usepackage{stfloats}
\usepackage{url}
\usepackage{verbatim}
\usepackage{graphicx}
\usepackage{xcolor}
\usepackage{cite}
\usepackage{booktabs}
\usepackage{multirow}
\usepackage{hyperref}
\usepackage{verbatim}
\theoremstyle{definition} 
\newtheorem{definition}{Definition}

\theoremstyle{remarkstyle}
\newtheorem{remark}{Remark}

\input{mysymbol.sty}

\def\bbx{{\ensuremath{\boldsymbol{x}}}}
\def\bbs{{\ensuremath{\boldsymbol{s}}}}

\def \unbbs {\underline{\bbs}}
\def \ovbbs {\overline{\bbs}}

\newtheorem{theorem}{Theorem}
\newtheorem{corollary}[theorem]{Corollary}
\newtheorem{proposition}{Proposition}
\begin{document}

\title{Matched Topological Subspace Detector}

\author{
    Chengen Liu,~\IEEEmembership{Student Member,~IEEE,} 
    Victor M. Tenorio,~\IEEEmembership{Student Member,~IEEE,} 
    Antonio G. Marques,~\IEEEmembership{Senior Member,~IEEE,} 
    and Elvin Isufi,~\IEEEmembership{Senior Member,~IEEE}

    \thanks{
        Chengen Liu and Elvin Isufi are with the Faculty of Electrical Engineering, Mathematics and Computer Science, Department of Intelligent Systems, Delft University of Technology (\{c.liu-15,e.isufi-1\}@tudelft.nl).
        
        Victor M. Tenorio and Antonio G. Marques are with the Department of Signal Theory and Communications, King Juan Carlos University (\{victor.tenorio,antonio.garcia.marques\}@urjc.es).
        
        
        
        A preliminary version of this work was presented in~\cite{liu2024hodge}. This paper is supported by the Dutch Grant GraSPA (No. 19497) financed by the Netherlands Organization for Scientific Research (NWO), by the Spanish AEI (AEI/10.13039/501100011033), grants PID2022-136887NB-I00, PID2023-149457OB-I00, and FPU20/05554, the Community of Madrid via IDEA-CM (TEC-2024/COM-89) and the Ellis Madrid Unit, and by the EU H2020 Grant Tailor (No 952215, agreements 76 and 82). Chengen Liu receives funding from the China Scholarship Council.
    }
}


\maketitle

\begin{abstract}
Topological spaces, represented by simplicial complexes, capture richer relationships than graphs by modeling interactions not only between nodes but also among higher-order entities, such as edges or triangles. This motivates the representation of information defined in irregular domains as topological signals. By leveraging the spectral dualities of Hodge and Dirac theory, practical topological signals often concentrate in specific spectral subspaces (e.g., gradient or curl). For instance, in a foreign currency exchange network, the exchange flow signals typically satisfy the arbitrage-free condition and hence are curl-free. However, the presence of anomalies can disrupt these conditions, causing the signals to deviate from such subspaces. In this work, we formulate a hypothesis testing framework to detect whether simplicial complex signals lie in specific subspaces in a principled and tractable manner. Concretely, we propose Neyman-Pearson matched topological subspace detectors for signals defined at a single simplicial level (such as edges) or jointly across all levels of a simplicial complex. The (energy-based projection) proposed detectors handle missing values, provide closed-form performance analysis, and effectively capture the unique topological properties of the data. We demonstrate the effectiveness of the proposed topological detectors on various real-world data, including foreign currency exchange networks.
\end{abstract}


\section{Introduction}
\IEEEPARstart{T}{opological} signals, such as those arising in simplicial complexes~\cite{bick2023higher,salnikov2018simplicial}, encode a more nuanced structure compared to graph signals by supporting multiway relationships among higher-order elements. Graph signals primarily focus on pairwise interactions between nodes, whereas topological signals can represent interactions among multiple entities simultaneously. In financial markets, for example, transactions may involve more than two companies at a time, and in protein molecules, the functional relationships may extend beyond simple binary interactions. Recent advances in signal processing and machine learning have introduced a variety of tools to handle topological signals~\cite{isufi2024topological,barbarossa2020topological,schaub2021signal}, including specialized convolutional and trend filtering techniques~\cite{yang2022simplicial,yang2022simplicialtrend}, neural networks~\cite{yang2022simplicialconvolutional,battiloro2024generalized}, Fourier analysis~\cite{barbarossa2020topological}, autoregressive models~\cite{krishnan2023simplicial}, signal recovery methods~\cite{reddy2024recovery}, and simplicial random walks~\cite{schaub2020random}.

The Hodge Laplacian provides an algebraic representation of topological structures and enables a spectral decomposition of simplicial signals~\cite{barbarossa2020topological,lim2020hodge,schaub2021signal}. Specifically, any simplicial signal of a given order can written as the sum of three orthogonal components, each lying on a subspace given by the decomposition of the Hodge Laplacian of that order. Focusing on edge signals, for instance, one can decompose them into three mutually orthogonal components: gradient, curl, and harmonic. Each component lives in a corresponding Hodge subspace and offers distinct insights into the nature of the signal~\cite{yang2022simplicial,isufi2022convolutional}. To \textit{jointly} incorporate signals defined at multiple orders (e.g., node, edge, and triangle signals), the Dirac operator~\cite{bianconi2021topological,baccini2022weighted} extends this idea, decomposing the entire space of simplicial signals into Dirac gradient, Dirac curl, and Dirac harmonic subspaces. 


These subspaces provide a better characterization of practical topological signals, as they often exhibit special properties such as being divergence-free or curl-free~\cite{liu2023unrolling}. A divergence-free signal implies that the inflow equals the outflow at each node, meaning there is no gradient component. For example, in traffic networks, where nodes represent intersections and edges correspond to roads, the traffic flow edge signal is nearly divergence-free, as vehicles entering a node will eventually exit it, assuming no congestion~\cite{yang2022simplicial,schaub2018flow}. Similarly, a curl-free signal implies that circulation within each triangle is zero. A notable example is found in the foreign exchange market, where nodes represent currencies and edges denote exchange possibilities. Under the arbitrage-free condition, the exchange rate edge flow is curl-free, ensuring that no profit can be obtained through a closed loop of transactions involving three currencies~\cite{jia2019graph}. However, when abnormalities occrr, such as noisy or incomplete measurements~\cite{jia2019graph,schaub2018flow} or adversarial attacks, these conditions no longer hold. In the case of traffic networks, congestion disrupts the divergence-free condition, introducing gradient components into the edge signal. Similarly, inaccuracies in exchange rate values violate the arbitrage-free condition, causing the edge signal to deviate from being curl-free.

To detect topological anomalies in a principled and mathematically tractable manner, we develop a topological matched subspace detection (MSD) framework inspired by the standard MSD approach~\cite{scharf1994matched}. MSD has a long and successful history across various applications, including communications~\cite{mccloud2000interference}, radar~\cite{rangaswamy2004robust}, and anomaly detection~\cite{stein2002anomaly}. MSD formulates the detection of a signal residing in a specific subspace as a hypothesis testing problem, leveraging the energy of the projected signal in the orthogonal complement of the target subspace. More recently, MSD has been applied to subgraph detection on graphs~\cite{chepuri2016subgraph}. However, existing graph-based MSD methods cannot effectively handle topological signals, which exhibit more intricate and intrinsic relationships among them.  
%
%
Additionally, prior works typically assume full availability of all signals, whereas in real-world applications, this assumption often does not hold. To address this limitation, we further extend the MSD framework to accommodate incomplete topological signals.


More specifically, we make the following contributions:

\begin{itemize}
\item [\text { 1) }]
We develop a topological MSD framework based on Hodge theory, generalizing MSD on graphs (node signals) without imposing any assumptions on the order of the underlying simplicial signal. More precisely, we formulate a hypothesis testing problem to determine whether a simplicial signal resides in a specific Hodge subspace. The test statistic for this detection task is derived using the generalized likelihood ratio test (GLRT), and its performance is characterized in closed form.
\item [\text { 2) }]
We extend topological MSD to jointly detect simplicial complex signals across all orders via Dirac theory. Additionally, we establish connections between the Hodge and Dirac MSD tasks, analyze their asymptotic performance, and demonstrate how joint signals can enhance Hodge-based detection tasks.
\item [\text { 3) }]
We address topological MSD in the presence of missing values. Specifically, we derive the optimal detector based on GLRT by projecting onto the subspace of interest. Furthermore, we analyze how the relationship between the dimension of the target subspace and the number of missing values leads to overdetermined and underdetermined cases.
\end{itemize}

The effectiveness of these detectors is validated through experiments on real-world datasets, including currency exchange markets, user-item interactions, water networks, and football games.

The remainder of the paper is organized as follows. Sec.~\ref{S:prelim} introduces preliminary concepts, while Sec.~\ref{S:problem} motivates and formulates the problem of interest. Sec.~\ref{S:complete} presents the MSD framework for both simplicial and simplicial complex signals based on Hodge and Dirac theory. Sec.~\ref{S:missing} discusses the optimal detector for scenarios with missing values. Sec.~\ref{S:exps} reports numerical experiments, and Sec.~\ref{S:conclusion} concludes the paper.

\section{Preliminaries} \label{S:prelim}
\subsection{Simplicial Complexes}
Let $\ccalV$ be a set containing $N_0$ vertices. Our goal is to define ${\mathcal P}^K$, which is a simplicial complex of order $K\leq N_0$ defined over $\ccalV$. To that end, we first introduce the $k$-simplex $\ccalW^k$, which is a set containing $k+1\leq N_0$ vertices of $\mathcal{V}$. Then, a simplicial complex $\ccalP^K$ of order $K$ is a collection of $k$-simplices (all defined over $\ccalV$ with $k=0,1,\ldots,K$) that satisfy the so-called ``inclusion property''. 
To be specific, let $N_k$ denote the number of $k$-simplices in ${\mathcal P}^K$. Then, the simplicial complex $\ccalP^K$ is formed by $\{\ccalW_n^0\}_{n=1}^{N_0}$, $\{\ccalW_n^1\}_{n=1}^{N_1}$, \dots, $\{\ccalW_n^K\}_{n=1}^{N_K}$, with $N=\sum_{k=0}^K N_k$ being the total number of simplices in $\ccalP^K$. Additionally, to satisfy the inclusion property, it must hold that for any $\ccalW_n^k \in \mathcal{P}^K$, all the subsets of $\ccalW_n^k$ are also part of the simplicial complex ${\mathcal P}^K$. To gain intuition, when embedding the simplicial complex in the Euclidean space, a 0-simplex corresponds to a node, a 1-simplex to an edge, and a 2-simplex to a triangle; see Fig.~\ref{fig:example}. The inclusion property implies that for a triangle to exist, all its edges and nodes must also be part of the simplicial complex. Additionally, it follows that a graph can be regarded as a simplicial complex of order $K=1$, as it contains only nodes and edges.

We consider the reference orientation of a simplex as the lexicographical ordering of the vertices, and represent the connections between different simplices by the incidence matrices $\mathbf{B}_k \in \mathbb{R}^{N_{k-1} \times N_k}$ describing the relationship between ($k$-1)-simplices and $k$-simplices~\cite{barbarossa2020topological}. Based on these incidence matrices, the structure of a simplicial complex can be represented by the \emph{Hodge Laplacian} matrices defined as
\begin{equation}\label{eq:hodge_lap}
\left\{\begin{aligned}
& \mathbf{L}_0=\mathbf{B}_1 \mathbf{B}_1^\top, \\
& \mathbf{L}_k=\underbrace{\mathbf{B}_k^\top \mathbf{B}_k}_{\mathbf{L}_{k,\ell}}+\underbrace{\mathbf{B}_{k+1} \mathbf{B}_{k+1}^\top}_{\mathbf{L}_{k,u}}, k=1, \ldots, K-1, \\
& \mathbf{L}_K=\mathbf{B}_K^\top \mathbf{B}_K.
\end{aligned}\right.
\end{equation}
Any \textit{intermediate} Laplacian matrix of order $k=1,\ldots,K-1$ contains two terms, which are the \emph{lower Laplacian} $\mathbf{L}_{k,\ell} = \mathbf{B}_k^\top \mathbf{B}_k$ and the \emph{upper Laplacian} $\mathbf{L}_{k,u} =\mathbf{B}_{k+1} \mathbf{B}_{k+1}^\top$. They encode respectively the lower adjacencies (e.g., two edges are adjacent via a common node) and upper adjacencies (e.g., two edges are adjacent by being the faces of the same triangle). For example, in Fig.~\ref{fig:example}, the edges $(1, 2)$ and $(2, 3)$ are lower adjacent, while the edges $(3, 4)$ and $(4, 5)$ are upper adjacent.

\subsection{Simplicial Signals}
Simplicial \textit{signals} are mappings from simplices to the set of real numbers. A $k-$simplicial signal, for short $k-$signal, $\boldsymbol{s}^k=\left[s_1^k, \ldots, s_{N_k}^k\right]^{\top} \in \mathbb{R}^{N_k}$ is a vector supported on $k$-simplices where each entry $s_n^k$ corresponds to the $n$-th $k$-simplex~\cite{barbarossa2020topological}. If the element $s_n^k$ is positive, the orientation of the signal is the same as the reference, and opposite otherwise. For example, in Fig.~\ref{fig:example}, the reference orientations of the $1-$simplices (edges) are denoted by the arrows.
A simplicial complex signal is defined as the concatenation of all $k-$signals
\begin{equation}\label{eq:simplicial_complexes_signal}
\boldsymbol{s}=\left[
\begin{matrix}
\boldsymbol{s}^0\\
\vdots \\
\boldsymbol{s}^K
\end{matrix}\right] \in \mathbb{R}^{N},
\end{equation}
where we recall that $N=\sum_{k=0}^K N_k$.

\subsection{Hodge Decomposition}
Hodge Laplacians admit a \emph{Hodge decomposition} stating that the space of $k-$signals can be decomposed into three orthogonal subspaces~\cite{lim2020hodge}
\begin{equation}\label{eq:hodge_decomposition}
\mathbb{R}^{N_k} \equiv \operatorname{span}\left(\mathbf{B}_k^\top\right) \oplus \operatorname{kernel}\left(\mathbf{L}_k\right) \oplus \operatorname{span}\left(\mathbf{B}_{k+1}\right)
\end{equation}
where $\oplus$ denotes the direct sum, and $\operatorname{span}$ and $\operatorname{kernel}$ denotes the column space and kernel (nullspace) of a matrix. It implies that any simplicial signal $\boldsymbol{s}^k$ of order $k$ can be expressed as a sum of three signals of order $k-1$, $k$ and $k+1$ fulfilling that, when multiplied by the respective incidence matrices as\footnote{Note that, as indicated by the use of a different notation, the induced signals $\unbbs^0$ and $\ovbbs^2$ in \eqref{eq:hodge_signal_decomp} and~\eqref{eq:edgeflow_decomp} are not the simplicial signals $\boldsymbol{s}^k$ that form the simplicial complex signal~\eqref{eq:simplicial_complexes_signal}.} 

\begin{equation}\label{eq:hodge_signal_decomp}
\boldsymbol{s}^k=\mathbf{B}_k^\top \unbbs^{k-1}+\boldsymbol{s}_\mathrm{H}^k+\mathbf{B}_{k+1} \ovbbs^{k+1},
\end{equation}
are orthogonal to each other. Here, the harmonic component $\boldsymbol{s}_\mathrm{H}^k \in \operatorname{kernel}\left(\mathbf{L}_k\right)$ is a solution of $\mathbf{L}_k \boldsymbol{s}_\mathrm{H}^k =\mathbf{0}$.
%

Without loss of generality, consider the edge space (1-signal) to illustrate the Hodge decomposition. The $\operatorname{span}\left(\mathbf{B}_1^\top\right)$, $\operatorname{span}\left(\mathbf{B}_{2}\right)$ and $\operatorname{kernel}\left(\mathbf{L}_1\right)$ are the gradient, the curl, and the harmonic subspace with dimension $N_{1,G}$, $N_{1,C}$ and $N_{1,H}$, respectively.
These subspaces have a direct connection with the eigenvectors of the corresponding Hodge Laplacian. More specifically, let us denote eigendecomposition of the Hodge Laplacian as
\begin{equation}\label{eq:first_order_lap}
\mathbf{L}_1=\mathbf{U}_1 \boldsymbol{\Lambda}_1 \mathbf{U}_1^{\top}
\end{equation}
where the column vectors of $\mathbf{U}_1\in \mathbb{R}^{N_1\times N_1}$ form an orthonormal basis, and $\boldsymbol{\Lambda}_1=\text{diag}(\lambda_1, \ldots , \lambda_{N_1}) \in \mathbb{R}^{N_1\times N_1}$ is a diagonal matrix containing the eigenvalues $\lambda_i$. The columns of the matrix $\mathbf{U}_1$ can be rearranged as $[\mathbf{U}_{1,\mathrm{G}}\;\mathbf{U}_{1,\mathrm{C}}\;\mathbf{U}_{1,\mathrm{H}}]$ where $\mathbf{U}_{1,\mathrm{G}}$, $\mathbf{U}_{1,\mathrm{C}}$ and $\mathbf{U}_{1,\mathrm{H}}$ collect the eigenvectors that span the gradient, curl and harmonic orthogonal subspaces~\cite{schaub2021signal}. Then, the Hodge decomposition implies that
\begin{equation}\label{eq:edgeflow_decomp}
\boldsymbol{s}^1=\boldsymbol{s}_{\mathrm{G}}^{1}+\boldsymbol{s}_\mathrm{H}^1+\boldsymbol{s}_{\mathrm{C}}^{1},\;\mathrm{with}\;\boldsymbol{s}_{\mathrm{G}}^{1}=\mathbf{B}_1^\top \unbbs^{0} \;\mathrm{and}\; \boldsymbol{s}_{\mathrm{C}}^{1}=\mathbf{B}_{2}
\end{equation}
where $\boldsymbol{s}_{\mathrm{G}}^{1}$, $\boldsymbol{s}_{\mathrm{C}}^{1}$ and $\boldsymbol{s}_\mathrm{H}^1$ are defined as the gradient, curl and harmonic component, respectively. The explanation of the subspace eigenvectors and the corresponding component (see also Fig. 1) are as follows:

\begin{itemize}
\item [\text { • }]
\textit{Gradient eigenvectors and gradient component}: the columns of $\mathbf{U}_{1,\mathrm{G}}\in \mathbb{R}^{N_1\times N_{1,G}}$ are the eigenvectors of $\mathbf{L}_{1,\ell}$ corresponding to the eigenvalues $\lambda_{\mathrm{G},i}>0$. The gradient component $\boldsymbol{s}_{\mathrm{G}}^{1}=\mathbf{B}_1^\top \boldsymbol{s}^{0}\in \operatorname{span} (\mathbf{B}_1^\top)$ is a  $1-$signal (edge signal) induced by a $0-$signal (node signal) and lives in the gradient space. It is computed by taking the difference between the node signal in the nodes connected by an edge. The projection of $\boldsymbol{s}^{1}$ onto the gradient subspace $\hat{\boldsymbol{s}}_{\mathrm{G}}^{1}=\mathbf{U}_{1,\mathrm{G}}^{\top} \boldsymbol{s}^{1}=\mathbf{U}_{1,\mathrm{G}}^{\top} \boldsymbol{s}_{\mathrm{G}}^{1} \in \mathbb{R}^{N_{{1,G}}}$ is the gradient embedding~\cite{yang2022simplicial}.
\item [\text { • }]
\textit{Curl eigenvectors and curl component}: the columns of $\mathbf{U}_{1,\mathrm{C}}\in \mathbb{R}^{N_1\times N_{1,C}}$ are the eigenvectors of $\mathbf{L}_{1,u}$ corresponding to the eigenvalues $\lambda_{\mathrm{C},i}>0$. The curl component $\boldsymbol{s}_{\mathrm{C}}^{1}=\mathbf{B}_{2} \boldsymbol{s}^{2}\in \operatorname{span} (\mathbf{B}_2)$ is an $1-$signal induced by a $2-$signal (triangle signal) and lives in the curl space. It is a local flow circulating along each triangle. The projection of $\boldsymbol{s}^{1}$ onto the curl subspace $\hat{\boldsymbol{s}}_{\mathrm{C}}^{1}=\mathbf{U}_{1,\mathrm{C}}^{\top} \boldsymbol{s}^{1}=\mathbf{U}_{1,\mathrm{C}}^{\top} \boldsymbol{s}_{\mathrm{C}}^{1} \in \mathbb{R}^{N_{{1,C}}}$ is the curl embedding~\cite{yang2022simplicial}.
\item [\text { • }]
\textit{Harmonic eigenvectors and harmonic component}: the columns of $\mathbf{U}_{\mathrm{H}}\in \mathbb{R}^{N_1\times N_{1,H}}$ are the eigenvectors of $\mathbf{L}_1$ corresponding to the zero eigenvalues $\lambda_{\mathrm{H},i}=0$. The harmonic component $\boldsymbol{s}_\mathrm{H}^1\in \operatorname{kernel}(\mathbf{L}_1)$ is an $1-$signal in the harmonic space $\operatorname{kernel}{(\mathbf{L}_1)}$ satisfying $\mathbf{L}_1 \boldsymbol{s}_\mathrm{H}^1=\mathbf{0}$. The projection of $\boldsymbol{s}^{1}$ onto the harmonic subspace $\hat{\boldsymbol{s}}_{\mathrm{H}}^{1}=\mathbf{U}_{1,\mathrm{H}}^{\top} \boldsymbol{s}^{1}=\mathbf{U}_{1,\mathrm{H}}^{\top} \boldsymbol{s}_{\mathrm{H}}^{1} \in \mathbb{R}^{N_{{1,H}}}$ is the harmonic embedding~\cite{yang2022simplicial}.
\end{itemize}

The projection of a specific component onto other Hodge subspaces is zero due to the orthogonality between different subspaces. This implicit and apparently simple property will play a major role in developing an MSD theory for topological signals. Two significant properties, which are common for real-world signals, stem from these three components:
\begin{itemize}
\item [\text { • }]
\textit{Curl-free}: $\operatorname{curl}(\boldsymbol{s}^{1})=\mathbf{B}_2^{\top} \boldsymbol{s}^{1} \in \reals^{N_2}$ is the curl operator which measures the curl of a $1-$signal (edge signal) $\boldsymbol{s}^{1}$. The $\emph{i}$th element of this vector represents the sum of the total flow circulating along the $\emph{i}$th triangle. An edge signal is curl-free if $\operatorname{curl}(\boldsymbol{s}^{1})=\mathbf{0}$. By definition, the gradient and harmonic components are curl-free. For instance, the currency exchange flow satisfying the arbitrage free condition is curl-free~\cite{jia2019graph}.
\item [\text { • }]
\textit{Divergence-free}: $\operatorname{div}(\boldsymbol{s}^{1})=\mathbf{B}_1 \boldsymbol{s}^{1} \in \reals^{N_0}$ is the divergence operator which measures the divergence of an edge signal $\boldsymbol{s}^{1}$. The $\emph{i}$th element of this vector represents the difference between the inflow and outflow at the $\emph{i}$th node. An edge signal is divergence-free if $\operatorname{div}(\boldsymbol{s}^{1})=\mathbf{0}$. By definition, the curl and harmonic components are divergence-free. For example, the Lastfm player transition flow is approximately divergence-free since the player is always switching between different artists~\cite{liu2023unrolling}.
\end{itemize}

\begin{figure}[t]
\centering
    \begin{minipage}[t]{0.24\linewidth}
        \centering
        \includegraphics[width=20mm]{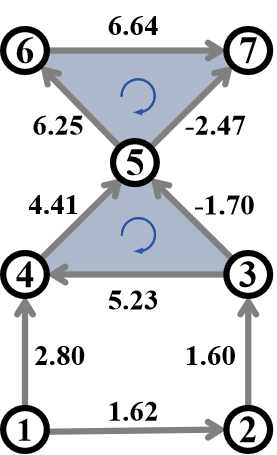}
        \centerline{(a) edge signal $\boldsymbol{s}^1$.}
    \end{minipage}%
    \begin{minipage}[t]{0.24\linewidth}
        \centering
        \includegraphics[width=20mm]{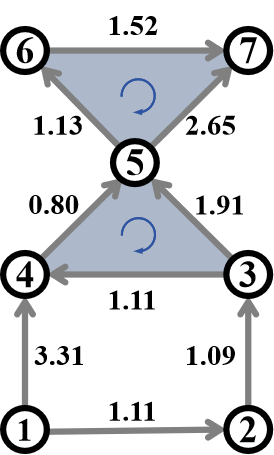}
        \centerline{(b) $\boldsymbol{s}_{\mathrm{G}}^{1}$.}
    \end{minipage}
    \begin{minipage}[t]{0.24\linewidth}
        \centering
        \includegraphics[width=20mm]{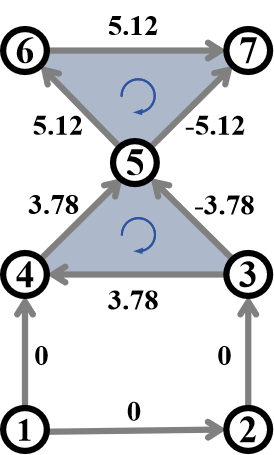}
        \centerline{(c) $\boldsymbol{s}_{\mathrm{C}}^{1}$.}
    \end{minipage}
    \begin{minipage}[t]{0.24\linewidth}
        \centering
        \includegraphics[width=20mm]{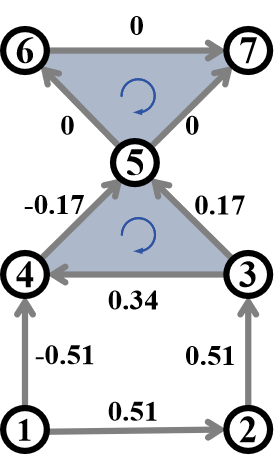}
        \centerline{(d) $\boldsymbol{s}_{\mathrm{H}}^{1}$.}
    \end{minipage}
\caption{Hodge decomposition of a $1-$signal (edge signal) on a simplicial complexes of order two. This edge signal can be decomposed into three different components: the gradient $\boldsymbol{s}_{\mathrm{G}}^{1}$, the curl $\boldsymbol{s}_{\mathrm{C}}^{1}$ and the harmonic component $\boldsymbol{s}_{\mathrm{H}}^{1}$.}
\label{fig:example}
\end{figure}

\subsection{Dirac Decomposition}
The Hodge decomposition limits the spectral processing to individual level simplicial signals. That is, it focuses on processing the $k-$signal using the spectrum of Laplacian $\mathbf{L}_k$, without taking into account the interrelationship between signals of varying orders. For a comprehensive approach to processing signals across all simplicial levels and utilizing their inter-simplicial connections, we can turn to the Dirac operator~\cite{calmon2022higher,baccini2022weighted}. Specifically, given a simplicial complex $\mathcal{P}^K$ of order $K$, the Dirac operator $\mathbf{D} \in \mathbb{R}^{N \times N}$ is defined as
\begin{equation}\label{eq:dirac}
\mathbf{D}=\!
\left[\begin{array}{ccccccc}
\mathbf{0} & \mathbf{B}_1 & \mathbf{0} & \cdots & \mathbf{0} & \mathbf{0} & \mathbf{0} \\
\mathbf{B}_1^\top & \mathbf{0} & \mathbf{B}_2 & \ddots & \mathbf{0} & \mathbf{0} & \mathbf{0} \\
\mathbf{0} & \mathbf{B}_2^\top & \mathbf{0} & \ddots & \mathbf{0} & \mathbf{0} & \mathbf{0} \\
\vdots     &    \ddots        & \ddots     & \ddots & \ddots     & \ddots &\vdots \\
\mathbf{0} & \mathbf{0} & \mathbf{0} & \ddots & \mathbf{0} & \mathbf{B}_{K-1} & \mathbf{0} \\
\mathbf{0} & \mathbf{0} & \mathbf{0} &   \ddots     & \mathbf{B}_{K-1}^\top & \mathbf{0} & \mathbf{B}_K \\
\mathbf{0} & \mathbf{0} & \mathbf{0} & \cdots & \mathbf{0} & \mathbf{B}_K^\top & \mathbf{0}
\end{array}\right]\!.
\end{equation}
The square of Dirac operator is a block diagonal matrix of the form $\mathbf{D}^2=\mathcal{L}= \operatorname{blkdiag}(\{\mathbf{L}_k\}_{k=0}^K)$, where $\operatorname{blkdiag}$ represents the block diagonal matrix whose diagonal is formed by the matrices $\{\mathbf{L}_k\}_{k=0}^K$.
%

To facilitate explanation, we focus next on simplicial complexes with an order of $K=2$.
The Dirac operator $\mathbf{D}$ can be broken down into $\mathbf{D}=\mathbf{D}_l+\mathbf{D}_u$, where 
\begin{equation}\label{eq:dirac_decomp}
\mathbf{D}_l=\left[\begin{array}{ccc}
\mathbf{0} & \mathbf{B}_1 & \mathbf{0} \\
\mathbf{B}_1^{\top} & \mathbf{0} & \mathbf{0} \\
\mathbf{0} & \mathbf{0} & \mathbf{0}
\end{array}\right], \; \mathbf{D}_u=\left[\begin{array}{ccc}
\mathbf{0} & \mathbf{0} & \mathbf{0} \\
\mathbf{0} & \mathbf{0} & \mathbf{B}_2 \\
\mathbf{0} & \mathbf{B}_2^{\top} & \mathbf{0}
\end{array}\right] .
\end{equation}
This implies that the space of the simplicial complex signals $\boldsymbol{s}=\left[\boldsymbol{s}^0\|\boldsymbol{s}^1\| \boldsymbol{s}^2\right] \in \mathbb{R}^N$ can be decomposed into three orthogonal subspaces, mirroring the scenario with $k-$signals and the Hodge Laplacian
\begin{equation}\label{eq:dirac_subspaces}
\mathbb{R}^N \equiv\operatorname{span}\left(\mathbf{D}_l\right) \oplus \operatorname{span}\left(\mathbf{D}_u\right) \oplus \operatorname{kernel}\left(\mathbf{D}\right)
\end{equation}
where $\operatorname{span}\left(\mathbf{D}_l\right)$ is the Dirac (or joint) gradient subspace considering the node potentials and the gradient flows jointly with dimension $N_G$; $\operatorname{span}\left(\mathbf{D}_u\right)$ is the Dirac (or joint) curl subspace considering the curl flows and the triangle potentials jointly with dimension $N_C$; and $\operatorname{kernel}\left(\mathbf{D}\right)$ is the Dirac (or joint) harmonic subspace with dimension $N_H$.
Thus, any simplicial complex signal $\boldsymbol{s}$ of order 2 can be expressed as a sum of three orthogonal signals
\begin{equation}\label{eq:dirac_signal_decomp}
\boldsymbol{s}=\boldsymbol{s}_\mathrm{G}+\boldsymbol{s}_\mathrm{C}+\boldsymbol{s}_\mathrm{H}
\end{equation}
where $\boldsymbol{s}_\mathrm{G} \in \operatorname{span}\left(\mathbf{D}_l\right)$, $\boldsymbol{s}_\mathrm{C} \in \operatorname{span}\left(\mathbf{D}_u\right)$ and $\boldsymbol{s}_\mathrm{H} \in \operatorname{kernel}(\mathbf{D})$. Therefore, the matrix of eigenvectors of $\mathbf{D}$ can be rearranged as
\begin{equation}\label{eq:dirac_reorder}
\mathbf{U}_\mathcal{P}=\left[\begin{array}{lll}
\mathbf{U}_{\mathcal{P}\mathrm{G}},\; & \mathbf{U}_{\mathcal{P}\mathrm{C}},\; & \mathbf{U}_{\mathcal{P}\mathrm{H}}
\end{array}\right]
\end{equation}
where $\mathbf{U}_{\mathcal{P}\mathrm{G}}\in \mathbb{R}^{N\times N_{G}}$ and $\mathbf{U}_{\mathcal{P}\mathrm{C}}\in \mathbb{R}^{N \times N_{C}}$ contain the non-zero eigenvectors of $\mathbf{D}_l$ and $\mathbf{D}_u$, respectively, and the columns of $\mathbf{U}_{\mathcal{P}\mathrm{H}}\in \mathbb{R}^{N \times N_{H}}$ span $\operatorname{kernel}(\mathbf{D})$. These matrices of eigenvectors can be computed from the singular vectors of the incidence matrices $\mathbf{B}_1$ and $\mathbf{B}_2$~\cite{calmon2022higher}.


\section{Problem formulation} \label{S:problem}
In practical scenarios, topological signals are often confined to specific topological subspaces, as indicated in equations~\eqref{eq:hodge_decomposition} or~\eqref{eq:dirac_subspaces}. This is particularly true for signals that are curl-free or divergence-free. However, anomalies do not follow this pattern; their signals typically span multiple subspaces. Consequently, identifying the subspaces to which a signal belongs is crucial for detecting anomalies or patterns in simplicial complex signals. The primary objective of this paper is to determine whether a simplicial complex signal $\boldsymbol{s}$ resides in certain topological subspaces, even when only noisy (and possibly incomplete) measurements are available.

More formally, let $\boldsymbol{x} = \boldsymbol{\Theta}(\boldsymbol{s} + \boldsymbol{n}) \in \mathbb{R}^{N_o}$ be the subset of noisy measurements, where $\boldsymbol{\Theta} \in \{0,1\}^{N_o \times N}$ is a sampling matrix with one 1 per row, selecting the $N_o$ available entries. Here, $\boldsymbol{s}$ denotes the noise-free and complete simplicial complex signal, and $\boldsymbol{n}$ is a zero-mean Gaussian noise vector $\boldsymbol{n} \sim \mathcal{N}\bigl(\boldsymbol{0}, \sigma^2 \mathbf{I}_{N}\bigr)$. Our problem of interest can be formulated as the following hypothesis test:
\begin{equation} \label{eq:hyp_test}
\begin{array}{l}
\mathcal{H}_0: \; \boldsymbol{s} \text{ resides within a specific topological subspace $\ccalS_\ccalP$}\\
\mathcal{H}_1: \; \boldsymbol{s} \text{ does not belong to $\ccalS_\ccalP$.}
\end{array}
\end{equation}
We address the hypothesis testing problem~\eqref{eq:hyp_test} using noisy and potentially incomplete data $\boldsymbol{x}\in \mathbb{R}^{N_o}$. When $\boldsymbol{\Theta} = \mathbf{I}$, we deal with complete data, as discussed in Section~\ref{S:complete}. Otherwise, we handle missing data, addressed in Section~\ref{S:missing}.

\section{Detection with complete signal} \label{S:complete}
In this section, we consider the simplicial detection task with complete data, i.e., $\boldsymbol{\Theta} = \bbI$. We begin by describing the Hodge subspace detection problem in Section~\ref{sub:hodge_complete}. We then extend our approach to the Dirac subspace detector in Section~\ref{sub:dirac_complete} and, in Section~\ref{sub:connections_hodge_dirac}, explore the relationships between the two.

\subsection{Hodge Subspace Detector} \label{sub:hodge_complete}
Consider that the $k$-simplicial signal $\boldsymbol{s}^{k}$ resides in a specific Hodge subspace, which can be written as a linear combination of the following eigenvectors:
\begin{equation}\label{eq:hodge_eigen}
\mathbf{U}_{\Delta} \in \{\mathbf{U}_{\mathrm{G}}, \mathbf{U}_{\mathrm{C}}, \mathbf{U}_{\mathrm{H}}, [\mathbf{U}_{\mathrm{G}},\mathbf{U}_{\mathrm{C}}], [\mathbf{U}_{\mathrm{G}},\mathbf{U}_{\mathrm{H}}], [\mathbf{U}_{\mathrm{C}},\mathbf{U}_{\mathrm{H}}]\}.
\end{equation}
The columns of $\bbU_{\Delta} \in \reals^{N_k \times N_{\Delta}}$ span the subspace of interest, which can be a combination of two of the Hodge subspaces eigenvectors such as $[\mathbf{U}_{\mathrm{G}},\mathbf{U}_{\mathrm{H}}]$. If $\boldsymbol{s}^k \in \operatorname{span} (\bbU_{\Delta})$, it is possible to write $\boldsymbol{s}^k = \bbU_{\Delta} \hat{\boldsymbol{s}}_{\Delta}^k$, with $\hat{\boldsymbol{s}}_{\Delta}^k \in \reals^{N_{\Delta}}$ containing the coefficients associated with each of the $N_{\Delta}$ vectors in the columns of $\bbU_{\Delta}$.

Likewise, we consider the complement (orthogonal) eigenvectors to $\mathbf{U}_{\Delta}$ which are the corresponding element of
\begin{equation}\label{eq:hodge_complement_eigen}
\mathbf{U}_{\overline{\Delta}} \in \{ [\mathbf{U}_{\mathrm{C}},\mathbf{U}_{\mathrm{H}}], [\mathbf{U}_{\mathrm{G}},\mathbf{U}_{\mathrm{H}}], [\mathbf{U}_{\mathrm{G}},\mathbf{U}_{\mathrm{C}}],\mathbf{U}_{\mathrm{H}}, \mathbf{U}_{\mathrm{C}}, \mathbf{U}_{\mathrm{G}}\}.
\end{equation}
whose $N_{\overline{\Delta}}$ columns span a complement Hodge subspace for $\boldsymbol{s}^{k}$. For instance, the foreign currency exchange rate flow tends to be curl-free and should align with the column space of $\bbU_{\Delta} = [\mathbf{U}_{\mathrm{G}},\mathbf{U}_{\mathrm{H}}]$, while the complement subspace would be $\mathbf{U}_{\overline{\Delta}} = \mathbf{U}_{\mathrm{C}}$.

Under this setting, the hypothesis testing problem is 
\begin{equation} \label{eq:hyp_hodge0}
\begin{array}{l}
\mathcal{H}_0: \; \boldsymbol{x}^{k} = \mathbf{U}_{\Delta}\hat{\boldsymbol{s}}^{k}_{\Delta} + \boldsymbol{n}^{k}\\
\mathcal{H}_1: \; \boldsymbol{x}^{k}= \mathbf{U}\hat{\boldsymbol{s}}^{k} + \boldsymbol{n}^{k},
\end{array}
\end{equation}
where $\hat{\boldsymbol{s}}^k \in \reals^{N_k}$ ($\hat{\boldsymbol{s}}^k_{\Delta} \in \reals^{N_\Delta}$) contains the coefficients associated with each of the eigenvectors of $\bbU$ ($\bbU_\Delta$). In essence, we assess whether the simplicial signal of interest $\boldsymbol{x}^k$ can be expressed as a linear combination of the columns of $\bbU_{\Delta}$ or if it contains a component beyond the subspace defined by those columns.

Multiplying both sides of~\eqref{eq:hyp_hodge0} by $\mathbf{U}_{\overline{\Delta}}^{\top}$ yields the projection of $\boldsymbol{x}^{k}$ onto the complement subspace $\hat{\boldsymbol{x}}^{k}_{\overline{\Delta}} =\mathbf{U}_{\overline{\Delta}}^{\top}\boldsymbol{s}^{k}+\mathbf{U}_{\overline{\Delta}}^{\top}\boldsymbol{n}^{k}=\hat{\boldsymbol{s}}^{k}_{\overline{\Delta}}+\hat{\boldsymbol{n}}^{k}_{\overline{\Delta}}$, with $\hat{\boldsymbol{s}}^{k}_{\overline{\Delta}}$ and $\hat{\boldsymbol{n}}^{k}_{\overline{\Delta}}$ representing the projections of the clean signal and noise onto the complement subspace, respectively. The projected noise satisfies $\hat{\boldsymbol{n}}^{k}_{\overline{\Delta}} \sim \mathcal{N}\left(\boldsymbol{0}_{N_{\overline{\Delta}}}, \sigma^2 \mathbf{I}_{N_{\overline{\Delta}}}\right)$.

Under hypothesis $\mathcal{H}_0$, the signal $\boldsymbol{s}^{k}$ lives in the Hodge subspace spanned by the columns of $\mathbf{U}_{\Delta}$ and the projection $\mathbf{U}_{\overline{\Delta}}^{\top}\boldsymbol{s}^{k}$ is $\boldsymbol{0}$ due to the orthogonality between the eigenvectors $\mathbf{U}_{\overline{\Delta}}^{\top}\mathbf{U}_{\Delta} = \boldsymbol{0}$. Thus, the projection of $\boldsymbol{x}^{k}$ onto the complement subspace under $\mathcal{H}_0$ is only noise $\hat{\boldsymbol{n}}^{k}_{\overline{\Delta}}$. Differently, under hypothesis $\mathcal{H}_1$, the projection is not only noise. Therefore, the hypothesis test takes the form
\begin{equation} \label{eq:hyp_hodge}
\begin{array}{l}
\mathcal{H}_0: \; \hat{\boldsymbol{x}}^{k}_{\overline{\Delta}} = \hat{\boldsymbol{n}}^{k}_{\overline{\Delta}}\\
\mathcal{H}_1: \; \hat{\boldsymbol{x}}^{k}_{\overline{\Delta}} =\mathbf{U}_{\overline{\Delta}}^{\top}\boldsymbol{s}^{k}+ \hat{\boldsymbol{n}}^{k}_{\overline{\Delta}}
\end{array}.
\end{equation}
This is a classical matched subspace detection problem when a signal is corrupted by noise~\cite{scharf1994matched}, in which we have to decide whether the projection onto the orthogonal subspace has a signal component or it is just noise. The problem of detecting deterministic signals with unknown parameters can be solved by the standard GLRT
\begin{equation} \label{eq:glrt}
T(\hat{\boldsymbol{x}}^{k}_{\overline{\Delta}})=\frac{p\left(\hat{\boldsymbol{x}}^{k}_{\overline{\Delta}} ;\hat{\boldsymbol{s}}_{\overline{\Delta} 1}^{k*}, \mathcal{H}_1\right)}{p\left(\hat{\boldsymbol{x}}^{k}_{\overline{\Delta}};\hat{\boldsymbol{s}}_{\overline{\Delta} 0}^{k*}, \mathcal{H}_0\right)} \underset{\mathcal{H}_0}{\stackrel{\mathcal{H}_1}{\gtrless}} \gamma
\end{equation}
where $p\left(\hat{\boldsymbol{x}}^{k}_{\overline{\Delta}};\hat{\boldsymbol{s}}_{\overline{\Delta} j}^{k*}, \mathcal{H}_j\right)$ is the probability density function (pdf) of $\hat{\boldsymbol{x}}^{k}_{\overline{\Delta}}$, $\hat{\boldsymbol{s}}_{\overline{\Delta} j}^{k*}$ is the maximum likelihood estimator (MLE) of $\hat{\boldsymbol{s}}^{k}_{\overline{\Delta}}$ under hypothesis $\mathcal{H}_j, j\in \{0, 1\}$ and $\gamma$ is the decision threshold.
When the test statistic $T(\hat{\boldsymbol{x}}^{k}_{\overline{\Delta}})$ exceeds (is below) the threshold $\gamma$, the detector determines that hypothesis $\mathcal{H}_1$ ($\mathcal{H}_0$) is true. Therefore, $\gamma$ controls the false-alarm and detection probabilities (the lower this threshold, the fewer times we will decide $\ccalH_0$ and viceversa).

Under a zero-mean Gaussian noise, the probability density function is
\begin{equation}\label{eq:gaussian_dis}
p\left(\hat{\boldsymbol{x}}^{k}_{\overline{\Delta}} ;\hat{\boldsymbol{s}}_{\overline{\Delta} j}^{k*}, \mathcal{H}_j\right)= \ccalN \left( \hat{\boldsymbol{x}}^{k}_{\overline{\Delta}}; \hat{\boldsymbol{s}}_{\overline{\Delta} j}^{k*}, \sigma^2 \bbI_{N_{\overline{\Delta}}} \right).
\end{equation}
Clearly, the MLE $\hat{\boldsymbol{s}}_{\overline{\Delta} j}^{k*}$ is $\hat{\boldsymbol{s}}_{\overline{\Delta}}^{k*}=\hat{\boldsymbol{x}}^{k}_{\overline{\Delta}}$ under hypothesis $\mathcal{H}_1$ and is $\hat{\boldsymbol{s}}_{\overline{\Delta}}^{k*}=\boldsymbol{0}$ under hypothesis $\mathcal{H}_0$. Thus the Hodge subspace detector becomes
\begin{equation} \label{eq:detector_hodge}
T(\hat{\boldsymbol{x}}^{k}_{\overline{\Delta}})=\|\hat{\boldsymbol{x}}^{k}_{\overline{\Delta}}\|_2^2  /\sigma^2\underset{\mathcal{H}_0}{\stackrel{\mathcal{H}_1}{\gtrless}} \gamma,
\end{equation}
which compares the SNR of the signal projected onto the column space of $\mathbf{U}_{\overline{\Delta}}$ with the threshold $\gamma$.

Given the Gaussian distribution of $\hat{\boldsymbol{x}}^{k}_{\overline{\Delta}}$, the test statistic $T(\hat{\boldsymbol{x}}^{k}_{\overline{\Delta}})$ in~\eqref{eq:detector_hodge} has a well-known a Chi-square distribution
\begin{equation}\label{eq:chi_sq}
T(\hat{\boldsymbol{x}}^{k}_{\overline{\Delta}}) \sim \begin{cases}\chi_{N_{\overline{\Delta}}}^2  & \text { under } \mathcal{H}_0 \\ \chi_{N_{\overline{\Delta}}}^2(\delta)& \text { under } \mathcal{H}_1\end{cases}
\end{equation}
where $N_{\overline{\Delta}}$ are the degrees of freedom and $\delta$ is a noncentrality parameter satisfying $\delta=\big\|\hat{\boldsymbol{s}}^{k}_{\overline{\Delta}}\big\|_2^2/\sigma^2$.
The higher the non-centrality parameter, the further apart from each other the distributions, and the easier the detection task. In the next section, we will see how considering simplicial complex signals under the Dirac setting gives a higher value of this parameter and therefore enhances the performance of the detector. Before that, we are in a position to characterize the performance of the Hodge detector. Given the distribution of the test statistic, the probability of false alarm is 
\begin{equation}\label{eq:pfa}
\text{P}_{\mathrm{FA}} \triangleq \operatorname{Pr}\{T(\hat{\boldsymbol{x}}^{k}_{\overline{\Delta}})>\gamma ; \mathcal{H}_0\}=Q_{\chi_{N_{\overline{\Delta}}}^2}\left(\gamma\right),
\end{equation}
and the probability of detection as
\begin{equation}\label{eq:pd}
\text{P}_{\text{D}}\triangleq \operatorname{Pr}\{T(\hat{\boldsymbol{x}}^{k}_{\overline{\Delta}})>\gamma ; \mathcal{H}_1\} =Q_{\chi_{N_{\overline{\Delta}}}^2(\delta)}(\gamma),
\end{equation}
where $Q_{\chi_{N_{\overline{\Delta}}}^2}\left(\cdot\right)$ is the right-tail probability function of the Chi-square distribution.


\subsection{Dirac Subspace Detector} \label{sub:dirac_complete}


Now, our objective is to determine whether the simplicial complex signal \(\boldsymbol{s}\) lies in certain Dirac subspaces spanned by any of the following eigenvectors:
\begin{equation}\label{eq:dirac_eigen}
\begin{aligned}
\mathbf{U}_{\mathcal{P}\Delta} \in \{ 
&\mathbf{U}_{\mathcal{P}\mathrm{G}}, 
\mathbf{U}_{\mathcal{P}\mathrm{C}}, 
\mathbf{U}_{\mathcal{P}\mathrm{H}}, [\mathbf{U}_{\mathcal{P}\mathrm{G}},\mathbf{U}_{\mathcal{P}\mathrm{C}}], \\
&
[\mathbf{U}_{\mathcal{P}\mathrm{G}},\mathbf{U}_{\mathcal{P}\mathrm{H}}], 
[\mathbf{U}_{\mathcal{P}\mathrm{C}},\mathbf{U}_{\mathcal{P}\mathrm{H}}]
\}.
\end{aligned}
\end{equation}
or any subcombination thereof\footnote{For example, if the simplicial signal has a sparse representation in the joint gradient subspace and in the joint curl subspace, then \(\mathbf{U}_{\mathcal{P}\Delta}\) could be built using only those eigenvectors.}. Analogous to the Hodge setting, let \(\mathbf{U}_{\mathcal{P}\overline{\Delta}}\) be the complement eigenvectors.

As in~\eqref{eq:hyp_hodge}, the hypothesis test for the Dirac setting can be restated as:
\begin{equation}\label{eq:hyp_test_dirac}
\begin{array}{l}
\mathcal{H}_0: \; \hat{\boldsymbol{x}}_{\overline{\Delta}} = \hat{\boldsymbol{n}}_{\overline{\Delta}}\\
\mathcal{H}_1: \; \hat{\boldsymbol{x}}_{\overline{\Delta}} =\mathbf{U}_{\mathcal{P}\overline{\Delta}}^{\top}\boldsymbol{s}+ \hat{\boldsymbol{n}}_{\overline{\Delta}}
\end{array}.
\end{equation}
With similar derivations, the Dirac subspace detector becomes:
\begin{equation} \label{eq:orth_energy_detector}
T(\hat{\boldsymbol{x}}_{\overline{\Delta}})=\|\hat{\boldsymbol{x}}_{\overline{\Delta}}\|_2^2  /\sigma^2\underset{\mathcal{H}_0}{\stackrel{\mathcal{H}_1}{\gtrless}} \gamma.
\end{equation}
Once again, we measure the SNR energy in the orthogonal subspace \(\operatorname{span}(\bbU_{\mathcal{P}\overline{\Delta}})\).

As in~\eqref{eq:detector_hodge}, the test statistic follows a Chi-square distribution, so the false alarm and detection probabilities match those in~\eqref{eq:pfa} and~\eqref{eq:pd}, respectively. In this case, under \(\ccalH_1\), the distribution’s non-centrality parameter is \(\delta = \|\hat{\boldsymbol{s}}_{\overline{\Delta}}\|_2^2 / \sigma^2\), where \(\hat{\boldsymbol{s}}_{\overline{\Delta}} = \mathbf{U}_{\mathcal{P}\overline{\Delta}}^{\top}\boldsymbol{s} \in \reals^{N_{\ccalP \overline{\Delta}}}\), and \(N_{\ccalP \overline{\Delta}}\) is the dimension of the Dirac complement subspace. Because a simplicial complex signal has a higher dimensionality than a \(k\)-simplicial signal (\(N > N_k\)), in the Dirac setting the energy of \(\hat{\boldsymbol{s}}_{\overline{\Delta}}\) is typically larger, leading to a higher non-centrality parameter and improved detection performance.

The detectors in~\eqref{eq:detector_hodge} and~\eqref{eq:orth_energy_detector} are energy detectors: they evaluate the signal energy in the orthogonal subspace, namely \(\operatorname{span}(\bbU_{\overline{\Delta}})\) versus \(\operatorname{span}(\bbU_{\mathcal{P}\overline{\Delta}})\). Given the probabilities of false alarm [cf.~\eqref{eq:pfa}] and detection [cf.~\eqref{eq:pd}], we can characterize the asymptotic behavior of~\eqref{eq:orth_energy_detector} following~\cite{kay1998fundamentals}, as stated next.

\begin{proposition}[Asymptotic performance] \label{prop:asymptotic}
    For a large dimension of complement subspace $N_{\ccalP \overline{\Delta}}$, the detection probability of the energy detector in~\eqref{eq:orth_energy_detector} is approximated by
\begin{equation}\label{eq:pd_energy}
\text{P}_{\text{D}} \approx Q\left(Q^{-1}\left(\text{P}_{\text{FA}}\right)-\sqrt{d^2}\right)
\end{equation}
where $Q(\cdot)$ is the right-tail probability function of the standard normal distribution and $d^2$ is the deflection coefficient defined as $d^2 = (\left\|\hat{\boldsymbol{s}}_{\overline{\Delta}}\right\|_2^2/\sigma^2)^2/2N_{\ccalP \overline{\Delta}}$.
\end{proposition}

\begin{proof}
    See Appendix~\ref{sub:proof_prop_asymptotic}.
\end{proof}

Equation~\eqref{eq:pd_energy} shows that the detection probability rises with the deflection coefficient $d^2$, given that the $Q$ function is monotonically decreasing. In turn, $d^2$ depends on: (i) the SNR of the projection of the signal onto the orthogonal subspace $\|\hat{\boldsymbol{s}}_{\overline{\Delta}}\|_2^2 / \sigma^2$; and (ii) the dimensionality of the orthogonal subspace $N_{\ccalP \overline{\Delta}}$. Consequently, the higher the projection energy under hypothesis $\mathcal{H}_1$, the greater the detection probability. Finally, note that the asymptotic performance in the Hodge scenario [c.f.~\eqref{eq:detector_hodge}] follows~\eqref{eq:pd_energy} by using $\hat{\boldsymbol{s}}_{\overline{\Delta}}^k$ instead of $\hat{\boldsymbol{s}}_{\overline{\Delta}}$ in the deflection coefficient.

\subsection{Connections Between Hodge and Dirac Detectors} \label{sub:connections_hodge_dirac}
Understanding how these two detectors relate enables us to exploit their connections and enhance the detection task. We summarize the connection in the following propositions.

\begin{proposition} \label{prop:hodge_zero}
    Let $\boldsymbol{s}^1$ be a 1-signal (edge signal).
    Also, let $\bbD$ denote the Dirac operator defined in~\eqref{eq:dirac}, whose decomposition as in~\eqref{eq:dirac_decomp} is $\bbD_l$ and $\bbD_u$. Finally, let $\bbB_1$, $\bbB_2$ be the node-to-edge and edge-to-triangle incidence matrices, respectively, and let $\bbL_1$ be the Laplacian matrix.
    Then, let $\boldsymbol{s} = \left[\boldsymbol{0}\|\boldsymbol{s}^1\| \boldsymbol{0}\right]$ be the corresponding simplicial complex signal. It holds that
    \begin{subequations} \label{eq:hodge_zero}
    \begin{align}
        \boldsymbol{s} &\in \operatorname{span}\left(\mathbf{D}_l\right) \Leftrightarrow \boldsymbol{s}^1 \in \operatorname{span}\left(\mathbf{B}_1^\top\right) \label{eq:hodge_zero1} \\
        \boldsymbol{s} &\in \operatorname{span}\left(\mathbf{D}_u\right) \Leftrightarrow \boldsymbol{s}^1 \in \operatorname{span}\left(\mathbf{B}_2\right) \label{eq:hodge_zero2} \\
        \boldsymbol{s} &\in \operatorname{kernel}\left(\mathbf{D}\right) \Leftrightarrow \boldsymbol{s}^1 \in \operatorname{kernel}\left(\mathbf{L}_1\right) \label{eq:hodge_zero3}
    \end{align}
    \end{subequations}
    where $\Leftrightarrow$ denotes necessary and sufficient conditions.
\end{proposition}


While the proof for Proposition~\ref{prop:hodge_zero} is omitted due to space limitations, it follows directly from the definitions of the Dirac operator and the incidence and Hodge Laplacian matrices.
This result indicates that detecting whether an edge signal $\boldsymbol{s}^1$ belongs to a particular Hodge subspace (or some subset thereof) is equivalent to detecting whether the simplicial complex signal whose node and triangle signals are zero-padded, $\boldsymbol{s} = \left[\boldsymbol{0}\|\boldsymbol{s}^1\| \boldsymbol{0}\right]$, lies in the corresponding Dirac subspace.

\begin{proposition} \label{prop:hodge_dirac}
    Let $\boldsymbol{s}^0$, $\boldsymbol{s}^1$, and $\boldsymbol{s}^2$ represent a 0-signal (node signal), 1-signal (edge signal), and 2-signal (triangle signal), respectively, and let $\boldsymbol{s} = \left[\boldsymbol{s}^0\|\boldsymbol{s}^1\| \boldsymbol{s}^2\right]$ be the corresponding simplicial complex signal.
    Also, let $\bbD$ denote the Dirac operator defined in~\eqref{eq:dirac}, whose decomposition as in~\eqref{eq:dirac_decomp} is $\bbD_l$ and $\bbD_u$. Finally, let $\bbB_1$, $\bbB_2$ be the node-to-edge and edge-to-triangle incidence matrices, respectively, and let $\bbL_1$ be the Laplacian matrix.
    Then, it holds that
    \begin{subequations} \label{eq:dirac_to_hodge}
    \begin{align}
    \boldsymbol{s}&\in \operatorname{span}\left(\mathbf{D}_l\right) \Rightarrow \boldsymbol{s}^1\in \operatorname{span}\left(\mathbf{B}_1^\top\right) \label{eq:dirac_to_hodge1}\\
    \boldsymbol{s}&\in \operatorname{span}\left(\mathbf{D}_u\right) \Rightarrow \boldsymbol{s}^1\in \operatorname{span}\left(\mathbf{B}_2\right) \label{eq:dirac_to_hodge2}\\
    \boldsymbol{s}&\in \operatorname{kernel}\left(\mathbf{D}\right) \Rightarrow \boldsymbol{s}^1\in \operatorname{kernel}\left(\mathbf{L}_1\right) \label{eq:dirac_to_hodge3}
    \end{align}
    \end{subequations}
    where $\Rightarrow$ denotes sufficient conditions.
\end{proposition}

\begin{proof}
    See Appendix~\ref{sub:proof_prop_nonzeropad}.
\end{proof}

Proposition~\ref{prop:hodge_dirac} indicates that simplicial signals of different orders can help detect the edge signal more effectively. To provide deeper insight, we consider the following simplified scenario. First, examine the Hodge setting under \(\ccalH_0\). The expected value of the test statistic in~\eqref{eq:detector_hodge} (assuming \(\sigma^2\) is absorbed into \(\gamma\)) is
\begin{equation}
\mathbb{E}[\|\hat{\boldsymbol{x}}^k\|_2^2] = \mathbb{E}[\|\hat{\boldsymbol{n}}_{\overline{\Delta}}\|_2^2] = N_{\overline{\Delta}} \sigma^2.
\end{equation}
In the Dirac setting, the expected value of the test statistic in~\eqref{eq:orth_energy_detector} is \(N_{\ccalP \overline{\Delta}} \sigma^2\). Conversely, under \(\ccalH_1\), we have
\begin{equation}
\mathbb{E}[\|\hat{\boldsymbol{x}}^k\|_2^2] = \mathbb{E}[\|\mathbf{U}_{\overline{\Delta}}^\top \boldsymbol{s}^k + \hat{\boldsymbol{n}}_{\overline{\Delta}}\|_2^2]
= \mathbb{E}[\|\mathbf{U}_{\overline{\Delta}}^\top \boldsymbol{s}^k\|_2^2] + N_{\overline{\Delta}} \sigma^2,
\end{equation}
If we assume the signal energy is proportional to its dimensionality, i.e., \(\mathbb{E}[\|\mathbf{U}_{\overline{\Delta}}^\top \boldsymbol{s}^k\|_2^2] = N_{\overline{\Delta}} \eta\), where \(\eta\) is a constant, then
\begin{equation}
\mathbb{E}[\|\hat{\boldsymbol{x}}^k\|_2^2] = N_{\overline{\Delta}} (\eta + \sigma^2).
\end{equation}
Assuming that \(\eta\) remains the same in the Dirac setting, it follows that
\(\mathbb{E}[\|\hat{\boldsymbol{x}}\|_2^2] = N_{\ccalP \overline{\Delta}} (\eta + \sigma^2)\).

When comparing the test statistic to a threshold, a useful measure of performance is the expected difference between the test statistic under \(\ccalH_1\) and \(\ccalH_0\). A larger difference implies an easier detection. In the Hodge case, this difference is
\begin{equation}
    \underbrace{N_{\overline{\Delta}} (\eta + \sigma^2)}_{\ccalH_1} - \underbrace{N_{\overline{\Delta}} \sigma^2}_{\ccalH_0} = N_{\overline{\Delta}} \eta,
\end{equation}
whereas in the Dirac case, the difference is \(N_{\ccalP \overline{\Delta}} \eta\). Since \(N_{\ccalP \overline{\Delta}} \geq N_{\overline{\Delta}}\), assuming identical noise power and signal energy under \(\ccalH_1\), the expected difference in the Dirac setting is larger, thereby facilitating detection. 


\section{Detection with missing data} \label{S:missing}

In the presence of missing values, the previous detectors do not hold because it is unclear whether the observed signal resides in the subspace of interest. In this section, we discuss the topological matched subspace detector for incomplete signals. For simplicity, we will focus on the Dirac subspace detection problem, as extending it to the Hodge setting is straightforward.

More formally, we have access only to a subset of entries selected by the sampling matrix $\bbTheta \neq \bbI$. The observed signal is defined as $\boldsymbol{x} = \boldsymbol{\Theta}(\boldsymbol{s}+\boldsymbol{n})\in \mathbb{R}^{N_o}$. The hypothesis testing problem can be reformulated as
\begin{equation}\label{eq:hyp_test_missing}
\begin{array}{l}
\mathcal{H}_0: \; \boldsymbol{x}= \mathbf{U}_{\mathcal{P}\Delta \Theta}\hat{\boldsymbol{s}}_{0}+\boldsymbol{n}_{\Theta}\\
\mathcal{H}_1: \; \boldsymbol{x}= \mathbf{U}_{\mathcal{P}\Theta}\hat{\boldsymbol{s}}_{1}+\boldsymbol{n}_{\Theta}
\end{array}
\end{equation}
where $\mathbf{U}_{\mathcal{P}\Delta \Theta} = \boldsymbol{\Theta} \mathbf{U}_{\mathcal{P}\Delta} \in \reals^{N_o \times N_{\ccalP {\Delta}}}$, $\mathbf{U}_{\mathcal{P}\Theta} = \boldsymbol{\Theta} \mathbf{U_\mathcal{P}}\in \reals^{N_o \times N}$ (i.e., the rows of the eigenvector matrices corresponding to the elements chosen by $\bbTheta$), $\hat{\boldsymbol{s}}_{0}$ and $\hat{\boldsymbol{s}}_{1}$ are the coefficients corresponding to the eigenvectors in $\mathbf{U}_{\mathcal{P}\Delta}$ and $\mathbf{U_\mathcal{P}}$, respectively, that construct the signal of interest. As before, we consider Gaussian noise $\boldsymbol{n}_{\Theta}=\boldsymbol{\Theta}\boldsymbol{n}\sim \mathcal{N}\bigl(\boldsymbol{0}, \sigma^2 \mathbf{I}_{N_o}\bigr)$.

Note that projecting onto the orthogonal subspace $\operatorname{span}(\bbU_{\mathcal{P}\overline{\Delta}})$ is not feasible, as $\mathbf{U}_{\mathcal{P}\overline{\Delta} \Theta}^\top \mathbf{U}_{\mathcal{P}\Delta \Theta} = \mathbf{U}_{\mathcal{P}\overline{\Delta}}^\top \boldsymbol{\Theta}^\top \boldsymbol{\Theta} \mathbf{U}_{\mathcal{P}\Delta} \neq \bbzero_N$, where $\bbzero_N$ is the $N \times N$ all-zero matrix.
We therefore formulate the GLRT by considering the distribution of $\boldsymbol{x}$ under each hypothesis, and by using the MLE of $\hat{\boldsymbol{s}}_j$, $j \in \{0,1\}$. The distribution of $\boldsymbol{x}$ under $\mathcal{H}_j$ is $\boldsymbol{x}\sim \mathcal{N}\bigl(\mathbf{U}_{\mathcal{P}\Theta,j}\hat{\boldsymbol{s}}_{j}^{*}, \sigma^2 \mathbf{I}_{N_o}\bigr)$, where $\mathbf{U}_{\mathcal{P}\Theta,0} = \mathbf{U}_{\mathcal{P}\Delta \Theta}$ under $\mathcal{H}_0$ and $\mathbf{U}_{\mathcal{P}\Theta,1} = \mathbf{U}_{\mathcal{P}\Theta}$ under $\mathcal{H}_1$. Hence, the detector becomes
\begin{equation} \label{eq:test_missing}
T(\boldsymbol{x}) = \frac{\|\boldsymbol{x}-\mathbf{U}_{\mathcal{P}\Delta \Theta}\hat{\boldsymbol{s}}_{0}^{*}\|_2^2-\|\boldsymbol{x}-\mathbf{U}_{\mathcal{P}\Theta}\hat{\boldsymbol{s}}_{1}^{*}\|_2^2}{\sigma^2}\underset{\mathcal{H}_0}{\stackrel{\mathcal{H}_1}{\gtrless}} \gamma.
\end{equation}
This detector measures the difference between the residual energies of the signals in the respective subspaces for each hypothesis. Specifically, the numerator in~\eqref{eq:test_missing} is the difference between two distinct terms, each capturing the energy of the discrepancy between the observed signal and its reconstruction via the eigenvectors of each hypothesis. Consequently, the difference between the missing-data detector~\eqref{eq:test_missing} and the Dirac subspace detector~\eqref{eq:orth_energy_detector} is that~\eqref{eq:test_missing} does not explicitly include the complement subspace but instead considers the subspace of interest for each hypothesis.

Finally, the observed signal also affects the MLE of $\hat{\bbs}_j$ for $j \in \{0,1\}$. This MLE depends on the relationship between the number of observed samples $N_o$ and the dimension of the subspace $N_{\ccalP {\Delta}}$. If $N_o > N_{\ccalP {\Delta}}$, we are in the overdetermined case; if $N_o \leq N_{\ccalP {\Delta}}$, we are in the underdetermined case. These two scenarios are detailed in the following sections.

\subsection{Overdetermined Case}

For the overdetermined case, we have $N_o > N_{\ccalP {\Delta}}$, and thus we find the MLE of $\hat{\boldsymbol{s}}_{j}$ by solving
\begin{flalign} \label{eq:mle_overdet}
\begin{array}{l} \hat{\boldsymbol{s}}_{j}^{*} = \underset{\hat{\boldsymbol{s}}_{j} }{\operatorname{argmin}} \;\|\boldsymbol{x}-\mathbf{U}_{\mathcal{P}\Theta,j}\hat{\boldsymbol{s}}_{j}\|_2^2.
\end{array}
\end{flalign}
Next, we substitute this value into the test statistic in~\eqref{eq:test_missing}, but before doing so, it is necessary to analyze the solution of this problem under both $\mathcal{H}_0$ and $\mathcal{H}_1$.

For the null hypothesis $\mathcal{H}_0$, the number of observations satisfies $N_o > N_{\ccalP {\Delta}}$, so $\mathbf{U}_{\mathcal{P}\Delta \Theta} = \boldsymbol{\Theta} \mathbf{U}_{\mathcal{P}\Delta} \in \reals^{N_o \times N_\Delta}$ is a tall matrix. The MLE of $\hat{\bbs}_0$ is thus given by the left pseudoinverse: $\hat{\boldsymbol{s}}_{0}^{*} = (\mathbf{U}_{\mathcal{P}\Delta \Theta})^\dagger \bbx_\Theta$. Since $\mathbf{U}_{\mathcal{P}\Delta \Theta} (\mathbf{U}_{\mathcal{P}\Delta \Theta})^{\dagger} \neq \bbI_{N_o}$, we have $\|\boldsymbol{x}-\mathbf{U}_{\mathcal{P}\Delta \Theta}\hat{\boldsymbol{s}}_{0}^{*}\|_2^2 \neq 0$.

Under the alternative hypothesis $\mathcal{H}_1$, $\mathbf{U}_{\mathcal{P}\Theta} = \boldsymbol{\Theta} \mathbf{U_\mathcal{P}} \in \reals^{N_o \times N}$ is a full row-rank, fat matrix, as it is formed by choosing $N_o \leq N$ rows from the full-rank $N \times N$ matrix $\mathbf{U_\mathcal{P}}$. One of the infinitely many solutions is obtained via the right pseudoinverse of $\mathbf{U}_{\mathcal{P}\Theta}$, yielding $\hat{\bbs}^*_1 = (\mathbf{U}_{\mathcal{P}\Theta})^{\dagger}\boldsymbol{x}$. Note that, because in this case $\mathbf{U}_{\mathcal{P}\Theta} (\mathbf{U}_{\mathcal{P}\Theta})^\dagger = \mathbf{I}$, it follows that $\|\bbx - \mathbf{U}_{\mathcal{P}\Theta} (\mathbf{U}_{\mathcal{P}\Theta})^\dagger \bbx\|_2^2 = 0$.

Substituting the estimates $\hat{\bbs}^*_0$ and $\hat{\bbs}^*_1$ back into~\eqref{eq:test_missing} results in the simplified detector
\begin{equation} \label{eq:detector_overdet_missing}
T(\boldsymbol{x}) = \frac{\|\boldsymbol{x}-\mathbf{U}_{\mathcal{P}\Delta \Theta}(\mathbf{U}_{\mathcal{P}\Delta \Theta})^\dagger\boldsymbol{x}\|_2^2}{\sigma^2}
\underset{\mathcal{H}_0}{\stackrel{\mathcal{H}_1}{\gtrless}} \gamma.
\end{equation}
Here, the matrix $\boldsymbol{P}_{\Delta \Theta} = \mathbf{U}_{\mathcal{P}\Delta \Theta}(\mathbf{U}_{\mathcal{P}\Delta \Theta})^\dagger$ is a projection operator onto the column space $\operatorname{span}(\mathbf{U}_{\mathcal{P}\Delta \Theta})$. Consequently, the proposed test statistic $T(\boldsymbol{x})$ measures the difference between $\boldsymbol{x}$ and its projection onto $\operatorname{span}(\mathbf{U}_{\mathcal{P}\Delta \Theta})$.

The transformed variable $ \boldsymbol{x} - \boldsymbol{P}_{\Delta \Theta} \boldsymbol{x} = (\bbI - \boldsymbol{P}_{\Delta \Theta}) \boldsymbol{x} = \boldsymbol{P}_{\overline{\Delta} \Theta} \boldsymbol{x} $ is the projection of $\boldsymbol{x}$ onto the orthogonal subspace of $\operatorname{span}(\bbU_{\ccalP\Delta \Theta})$\footnote{In a slight abuse of notation, we let $\boldsymbol{P}_{\overline{\Delta} \Theta} = \bbI - \boldsymbol{P}_{\Delta \Theta}$ denote the projection operator onto the orthogonal subspace of $\operatorname{span}(\bbU_{\ccalP\Delta \Theta})$, even though it is not strictly a projection onto $\operatorname{span}(\bbU_{\ccalP \overline{\Delta} \Theta})$. Its role depends on the entries chosen by $\bbTheta$. The two subspaces (orthogonal to $\operatorname{span}(\bbU_{\ccalP\Delta \Theta})$ and $\operatorname{span}(\bbU_{\ccalP \overline{\Delta} \Theta})$) coincide only when no data is missing, i.e., $\bbTheta=\bbI$.}. This variable still follows a Gaussian distribution with covariance matrix $\sigma^2 \boldsymbol{P}_{\overline{\Delta} \Theta}^{\top} \boldsymbol{P}_{\overline{\Delta} \Theta} = \sigma^2 \boldsymbol{P}_{\overline{\Delta} \Theta}$ (since the projection matrix is symmetric and idempotent). Its norm follows a Chi-square distribution with $\operatorname{tr}(\boldsymbol{P}_{\overline{\Delta} \Theta}) = N - \operatorname{rank}(\bbU_{\ccalP\Delta \Theta})$ degrees of freedom~\cite{Olkin1992QuadraticFI}, where $\operatorname{tr}$ is the trace operator and $\operatorname{rank}$ returns the rank of the matrix. Using the fact that a projection matrix has exactly one eigenvalue per dimension of the subspace it projects onto (and zeros for the rest), if $\bbU_{\ccalP\Delta \Theta}$ is full column rank (i.e., $\operatorname{rank} (\bbU_{\ccalP\Delta \Theta}) = N_{\ccalP \Delta}$), the degrees of freedom of the Chi-square distribution become $N - N_{\ccalP {\Delta}} = N_{\ccalP \overline{\Delta}}$. The associated false alarm and detection probabilities are given by~\eqref{eq:pfa} and~\eqref{eq:pd}, respectively, with the non-centrality parameter under $\ccalH_1$ being $\delta = \| \boldsymbol{P}_{\overline{\Delta} \Theta} \mathbf{U}_{\mathcal{P}\Theta}\hat{\boldsymbol{s}}_{1} \|_2^2 / \sigma^2$.

\begin{remark}
When no data is missing, the sampling matrix $\boldsymbol{\Theta}$ is the identity. In this case, detector~\eqref{eq:detector_overdet_missing} simplifies to $T(\boldsymbol{x}) = (\|\boldsymbol{x}\|_2^2 - \|\mathbf{U}_{\mathcal{P}\Delta}^\top\boldsymbol{x}\|_2^2)/\sigma^2$, which is equivalent to detector~\eqref{eq:orth_energy_detector}. Indeed, $\|\boldsymbol{x}\|_2^2 - \|\mathbf{U}_{\mathcal{P}\Delta}^\top\boldsymbol{x}\|_2^2$ is the energy of the projection of the signal onto the complement subspace, $\|\mathbf{U}_{\mathcal{P}\overline{\Delta}}^\top\boldsymbol{x}\|_2^2$, as dictated by Parseval's theorem.
\end{remark}

\smallskip
\noindent\textbf{Connections to projection detectors.} The GLRT topological detector in~\eqref{eq:detector_overdet_missing} is equivalent to the projection detector proposed in~\cite[Section 5]{balzano2010high} when $\bbU_{\mathcal{P}\Delta \Theta}$ is a full column rank matrix, as stated in Proposition~\ref{prop:other_missing_analsis}. Under these conditions, we can adapt the results of~\cite{balzano2010high} to probabilistically characterize the performance of~\eqref{eq:detector_overdet_missing} in comparison to the scenario with no missing values. Let $\operatorname{span}\left(\mathbf{U}_{\mathcal{P}\Delta}\right)$ denote the subspace spanned by the columns of $\mathbf{U}_{\mathcal{P}\Delta}$ (of dimension $N_{\ccalP {\Delta}}$), and let $\operatorname{span}\left(\mathbf{U}_{\mathcal{P}\overline{\Delta}}\right)$ denote the orthogonal subspace spanned by the columns of $\mathbf{U}_{\mathcal{P}\overline{\Delta}}$ (of dimension $N_{\ccalP \overline{\Delta}}$).

\begin{proposition}\label{prop:other_missing_analsis}
    Let $\bbU_{\mathcal{P}\Delta \Theta} = \bbTheta \bbU_{\mathcal{P}\Delta}$ be the rows of the eigenvectors matrix $\bbU_{\mathcal{P}\Delta}$ selected by the sampling matrix $\bbTheta$. Also, let $\bbx$ be the elements of the sampled simplicial complex signal. Finally, let $\boldsymbol{P}_{\Delta \Theta} = \mathbf{U}_{\mathcal{P}\Delta \Theta}(\mathbf{U}_{\mathcal{P}\Delta \Theta}^{\top} \mathbf{U}_{\mathcal{P}\Delta \Theta})^\dagger \mathbf{U}_{\mathcal{P}\Delta \Theta}^{\top}$ be the projection operator onto the subspace $\operatorname{span}(\mathbf{U}_{\mathcal{P}\Delta \Theta})$. Assuming that $\bbU_{\mathcal{P}\Delta \Theta}$ is full column rank, the detector given in~\eqref{eq:detector_overdet_missing} is equivalent to the detector
    \begin{equation}\label{eq:other_missing_detector}
        T(\boldsymbol{x}) = \left\|\boldsymbol{x}-\boldsymbol{P}_{\Delta \Theta}\boldsymbol{x}\right\|_2^2\underset{\mathcal{H}_0}{\stackrel{\mathcal{H}_1}{\gtrless}} \gamma,
    \end{equation}
    proposed in~\cite[Section 5]{balzano2010high}.
\end{proposition}

\begin{proof}
    By using the definition of the left pseudoinverse for a full column rank matrix $\mathbf{A}^\dagger = (\mathbf{A}^\top \mathbf{A})^{-1} \mathbf{A}^\top$, we have that
    \begin{align*}
        \mathbf{U}_{\mathcal{P}\Delta \Theta}(\mathbf{U}_{\mathcal{P}\Delta \Theta})^\dagger &= \mathbf{U}_{\mathcal{P}\Delta \Theta} (\mathbf{U}_{\mathcal{P}\Delta \Theta}^\top\mathbf{U}_{\mathcal{P}\Delta \Theta})^{-1} \mathbf{U}_{\mathcal{P}\Delta \Theta}^\top \\
        & = \mathbf{U}_{\mathcal{P}\Delta \Theta} (\mathbf{U}_{\mathcal{P}\Delta \Theta}^\top\mathbf{U}_{\mathcal{P}\Delta \Theta})^\dagger \mathbf{U}_{\mathcal{P}\Delta \Theta}^\top = \boldsymbol{P}_{\Delta \Theta},
    \end{align*}
    where in the second equality we used the fact that, as $\mathbf{U}_{\mathcal{P}\Delta \Theta}^\top\mathbf{U}_{\mathcal{P}\Delta \Theta}$ is a full rank square matrix (owing to $\mathbf{U}_{\mathcal{P}\Delta \Theta}$ being full column rank), its pseudo-inverse and inverse coincide.
    Moreover, the noise power $\sigma^2$ in the denominator of~\eqref{eq:detector_overdet_missing} is absorbed into the threshold $\gamma$ on the right-hand side, making the two detectors equivalent.
\end{proof}

We are ready to probabilistically characterize the performance of the detector, but first we define the coherence of a subspace:

\begin{definition}[Subspace coherence~\cite{candes2012exact}]
The coherence of an $R$-dimensional subspace $\ccalS$ is defined as
\begin{equation}\label{eq:coherence}
\mu(\ccalS) := \frac{N}{R} \max_j \left\| \boldsymbol{P}_\ccalS \boldsymbol{e}_j \right\|_2^2,
\end{equation}
where $\boldsymbol{P}_\ccalS$ is the projection operator onto $\ccalS$, $\boldsymbol{e}_j$ is the $j$th standard basis element, and $N$ is the signal dimension.
\end{definition}

For a vector $\boldsymbol{v}$, $\mu(\boldsymbol{v})$ denotes the coherence of the subspace spanned by $\boldsymbol{v}$. We now claim the following about detector~\eqref{eq:detector_overdet_missing}.

\begin{corollary}\label{coro:other_analsis}
Define the decomposition of the signal $\boldsymbol{x}$ as $\boldsymbol{x} = \boldsymbol{x}_{\Delta} + \boldsymbol{x}_{\overline{\Delta}} \in \reals^{N}$, where $\boldsymbol{x}_{\Delta}\in \operatorname{span}\left(\mathbf{U}_{\mathcal{P}{\Delta}}\right)$ and $\boldsymbol{x}_{\overline{\Delta}}\in \operatorname{span}\left(\mathbf{U}_{\mathcal{P}\overline{\Delta}}\right)$. Let $\epsilon>0$ be a constant and assume $N_o \geq \frac{8}{3} N_{\ccalP {\Delta}} \mu(\operatorname{span}\left(\mathbf{U}_{\mathcal{P}{\Delta}}\right)) \log \left(\frac{2 N_{\ccalP {\Delta}}}{\epsilon}\right)$. Then, with probability at least $1 - 4\epsilon$,
\begin{equation} \label{eq:prop1}
\alpha \!\left\|\boldsymbol{x}\!-\!\boldsymbol{P}_{{\Delta}} \boldsymbol{x}\right\|_2^2 \!\leq \!\left\|\boldsymbol{x}\!-\!\boldsymbol{P}_{{\Delta}{\Theta}} \boldsymbol{x}\right\|_2^2
\!\leq\!(1\!+\!\beta) \frac{N_o}{N} \left\|\boldsymbol{x}\!-\!\boldsymbol{P}_{{\Delta}} \boldsymbol{x}\right\|_2^2
\vspace{0.4cm}
\end{equation}
where $\delta=\sqrt{\frac{8 N_{\ccalP {\Delta}} \mu(\operatorname{span}\left(\mathbf{U}_{\mathcal{P}{\Delta}}\right))}{3 N_o} \log \left(\frac{2 N_{\ccalP {\Delta}}}{\epsilon}\right)}$, $\gamma=\sqrt{2 \mu(\boldsymbol{x}_{\overline{\Delta}}) \log \left(\frac{1}{\epsilon}\right)}$, $\alpha=\frac{N_o (1-\beta)-N_{\ccalP {\Delta}} \mu(\operatorname{span}\left(\mathbf{U}_{\mathcal{P}{\Delta}}\right)) \frac{(1+\gamma)^2}{(1-\delta)}}{N}$, and $\beta=\sqrt{\frac{2 \mu(\boldsymbol{x}_{\overline{\Delta}})^2}{N_o} \log \left(\frac{1}{\epsilon}\right)}$.
\end{corollary}

\begin{proof}
    The probabilistic bounds in \eqref{eq:prop1} follow by applying the proof of~\cite[Theorem 1]{balzano2010high}.
\end{proof}

The result provided in Proposition~\ref{prop:other_missing_analsis} indicates that, under the assumption that $\mathbf{U}_{\mathcal{P}\Delta \Theta}$ is full column rank, the GLRT-based detector for simplicial complex signals is equivalent to the detector proposed for missing data in~\cite{balzano2010high}. This requirement is the same as the one in~\cite[Th. 1]{chen15sampling} for perfect recovery under sampling, essentially stipulating that the $N_o$ observed rows of $\mathbf{U}_{\mathcal{P}\Delta}$ span $\reals^{N_\Delta}$. Otherwise, the problem falls into the underdetermined setting, discussed in the next section.

When $\beta$, $\gamma$, and $\delta$ are close to zero, the lower bound in $\|\boldsymbol{x}-\boldsymbol{P}_{{\Delta}{\Theta}} \boldsymbol{x}\|_2^2$ is approximately
\begin{equation}
\frac{N_o-N_{\ccalP {\Delta}} \mu(\operatorname{span}\left(\mathbf{U}_{\mathcal{P}{\Delta}}\right))}{N}\left\|\boldsymbol{x}-\boldsymbol{P}_{{\Delta}} \boldsymbol{x}\right\|_2^2.
\end{equation}
This arises when, for instance, $N_o$ is large or the subspace dimension $N_{\ccalP {\Delta}}$ is small. Because the coherence of $\operatorname{span}\left(\mathbf{U}_{\mathcal{P}{\Delta}}\right)$ is bounded by $1 \leq \mu(\operatorname{span}\left(\mathbf{U}_{\mathcal{P}{\Delta}}\right)) \leq \frac{N}{N_{\ccalP {\Delta}}}$, if $N_o \leq N_{\ccalP {\Delta}}$, the lower bound might always be zero or negative even if $\|\boldsymbol{x}-\boldsymbol{P}_{{\Delta}} \boldsymbol{x}\|_2^2 \geq 0$. Hence, the performance of the detector in~\eqref{eq:detector_overdet_missing} will be poor with high probability. This underscores that, for the proposed detector to function effectively, we need at least $N_{\ccalP {\Delta}}$ observations—that is, the dimension of the subspace we aim to detect.

\subsection{Underdetermined Case}

Now we deal with the case of having fewer observations than the subspace dimension, i.e., $N_o \leq N_{\ccalP {\Delta}}$. The fact that, for $\ccalH_1$, $\mathbf{U}_{\mathcal{P} \Theta} \left(\mathbf{U}_{\mathcal{P} \Theta}\right)^{\dagger} = \bbI_{N_o}$ and thus $\| \bbx - \mathbf{U}_{\mathcal{P}\Theta} (\mathbf{U}_{\mathcal{P}\Theta})^\dagger \bbx\|_2^2 = 0$ still holds if we estimate the MLE $\hat{\boldsymbol{s}}^*_1$ by solving~\eqref{eq:mle_overdet}. However, under $\mathcal{H}_0$, $\mathbf{U}_{\mathcal{P}\Theta,0}=\mathbf{U}_{\mathcal{P}\Delta \Theta}$ is a fat matrix. If this matrix is full row rank and we obtain the MLE of $\hat{\bbs}_0$ via~\eqref{eq:mle_overdet}, the solution is non-unique, and setting $\hat{\boldsymbol{s}}_0^{*} = \left(\mathbf{U}_{\mathcal{P}\Delta \Theta}\right)^{\dagger} \boldsymbol{x}$ makes detection impossible. This arises because, if $\mathbf{U}_{\mathcal{P}\Delta \Theta}$ is full row rank, its columns span $\reals^{N_o}$, implying that $\bbx \in \operatorname{span}(\mathbf{U}_{\mathcal{P}\Delta \Theta}) \equiv \reals^{N_o}$ in every scenario. Consequently, for our proposed detector, $\mathbf{U}_{\mathcal{P}\Delta \Theta} \left(\mathbf{U}_{\mathcal{P}\Delta \Theta}\right)^{\dagger} = \bbI_{N_o}$, which always yields $T(\bbx) = \| \bbx - \mathbf{U}_{\mathcal{P}\Delta \Theta}\left(\mathbf{U}_{\mathcal{P}\Delta \Theta}\right)^{\dagger}\bbx \|_2^2 = 0$ in \eqref{eq:detector_overdet_missing}.

For this more challenging case, we can employ a regularized version of the detector and obtain the MLEs of both $\hat{\bbs}_0$ and $\hat{\bbs}_1$ by solving
\begin{flalign} \label{eq:mle_underdet}
\begin{array}{l} \hat{\boldsymbol{s}}_{j}^{*} = \underset{\hat{\boldsymbol{s}}_{j} }{\operatorname{argmin}} \;\|\boldsymbol{x}-\mathbf{U}_{\mathcal{P}\Theta,j}\hat{\boldsymbol{s}}_{j}\|_2^2 + \lambda_j \Omega (\hat{\boldsymbol{s}}_{j})
\end{array}
\end{flalign}
where $\Omega (\hat{\boldsymbol{s}}_{j})$ leverages prior information about $\hat{\boldsymbol{s}}_{j}$, such as it being low-pass or sparse. For instance, if the simplicial embedding $\hat{\boldsymbol{s}}_{j}$ is low-pass, then we can set the regularizer $\lambda_j \Omega (\hat{\boldsymbol{s}}_{j})$ in equation~\eqref{eq:mle_underdet} to $\lambda_j\|\mathbf{R}_j \hat{\boldsymbol{s}}_{j}\|_2^2$, where $\mathbf{R}_j$ is a diagonal matrix with decreasing diagonal entries. A closed-form solution may exist depending on the regularization $\Omega$; otherwise, a numerical estimate can be obtained.

\section{Numerical results} \label{S:exps}
We corroborate the proposed detectors with numerical experiments on four different real-world datasets. In Sec.~\ref{sub:datasets}, we introduce the datasets, whereas in Sec.~\ref{sub:experi_hodge}, we evaluate the performance of the Hodge subspace detectors (HSD). In Sec.~\ref{sub:experi_dirac}, we evaluate the Dirac subspace detectors (DSD). Finally, in Sec.~\ref{sub:incomplete}, we assess the impact of having incomplete data.

\begin{figure}[t]
    \centering
    \includegraphics[width=0.4\textwidth]{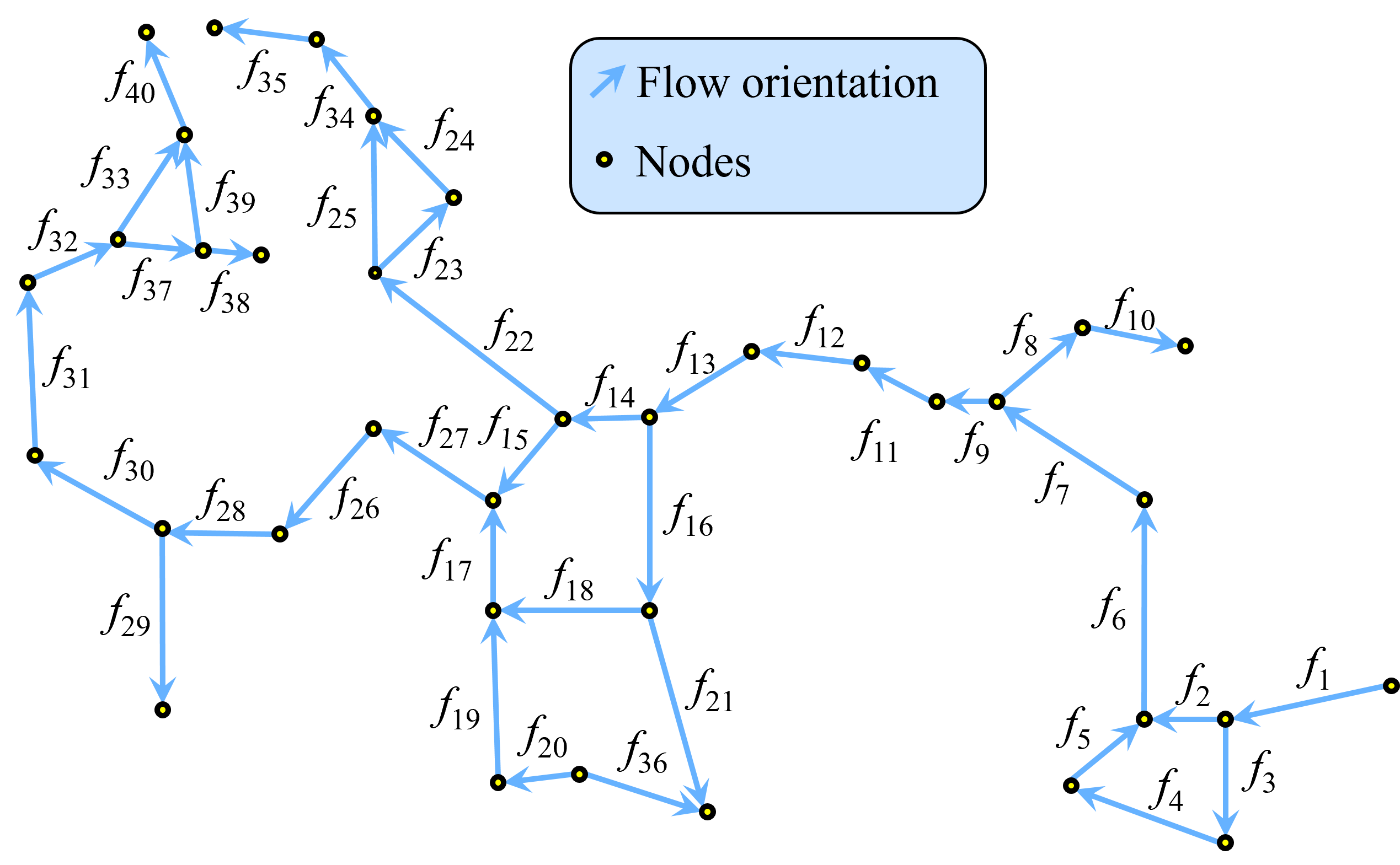}
    \caption{Cherry hills water network which has 36 nodes represent tanks, 40 edges represent the pipes and 2 triangles represents the areas enclosed by three pipes. Different water demands generates different water flow rate over the edges and water pressure over the nodes.}
    \label{fig:cherry}
\end{figure}

\subsection{Datasets}\label{sub:datasets}
We use four datasets, summarized in Table~\ref{tab:properties_datasets}:

\subsubsection{Forex~\cite{jia2019graph}}
This dataset represents foreign currency exchanges, where each currency is a node, pairwise exchanges between two currencies are treated as edges, and any three currencies form a triangle. The edge signal is the logarithm of the exchange rate. To ensure the exchange rates are arbitrage-free  --meaning no profit can be obtained by trading currencies in a loop-- the rates must balance in any cyclical exchange. For example, starting with currency A, converting to B, then to C, and finally back to A, should yield no net gain. Denoting the exchange rate between A and B as $r^{\mathrm{A}/\mathrm{B}}$, the arbitrage-free condition can be written as $r^{\mathrm{A}/\mathrm{B}}\,r^{\mathrm{B}/\mathrm{C}} = r^{\mathrm{A}/\mathrm{C}}$. Taking the logarithm of the exchange rate, defined as $\hat{r}^{\mathrm{A}/\mathrm{B}} = \log \bigl(r^{\mathrm{A}/\mathrm{B}}\bigr)$, we obtain $\hat{r}^{\mathrm{A}/\mathrm{B}} + \hat{r}^{\mathrm{B}/\mathrm{C}} = \hat{r}^{\mathrm{A}/\mathrm{C}}$, indicating that the edge signal is curl-free.

\subsubsection{Lastfm~\cite{jia2019graph}}
This dataset records instances when a user switches from one artist to another on a music player. Each artist is represented as a node, and an edge is created between two artists whenever a user switches from one artist to the other. Any triangle formed by three edges is treated as filled. The edge signal is built as follows: each time a user switches from artist $A$ to $B$, a unit is added to the edge signal from $A$ to $B$. Since users consistently switch to another artist after listening to one, except for the initial and terminal nodes, the divergence at other nodes is zero. Consequently, the edge signal is approximately divergence-free.
\begin{table*}[t]
\centering
\caption{{Properties of the datasets.}}
\begin{tabular}{llcccccccccc}
   \toprule
   \textbf{{Datasets}} &{Nodes}&{Edges}&{Triangles}&{edge signal property}&$N_{{\Delta}}$&$N_{ \overline{\Delta}}$&$N_{\ccalP {\Delta}}$&$N_{\ccalP \overline{\Delta}}$ \\
   \midrule
   \textbf{{Forex}~\cite{jia2019graph}}&{25}&{300}&{2300}&{curl-free}&24&276&48&2577 \\
   \textbf{{Lastfm}~\cite{jia2019graph}} &{657}&{1997}&{1276}&{divergence-free}&1341&656&2618&1312 \\
   \textbf{{Cherry hills}~\cite{krishnan2023simplicial}} &{36}&{40}&{2}&{curl-free}&38&2&74&4 \\
   \textbf{{Football}} &{11}&{55}&{165}&{divergence-free}&45&10&211&20 \\
   \bottomrule
\end{tabular}
\label{tab:properties_datasets}
\end{table*}

\subsubsection{Cherry hills~\cite{krishnan2023simplicial}}
This dataset represents a water distribution network, where each node corresponds to a tank, each edge to a pipe transporting water, and each triangle to an area enclosed by three pipes. The node signal is the water pressure at each tank (in pounds per square inch, scaled by $10^{-4}$), the edge signal captures water flow rate through each pipe (in cubic feet per second), and the triangle signal is the sum of the water demand across the three nodes forming the triangle (in cubic feet per second). The edge flow signals are generated with the EPANET software under a demand-driven model~\cite{krishnan2023simplicial}, where varying demands lead to different flow rates. The dataset comprises 55 hours of recorded edge signal and node pressure data, sampled hourly, with hourly averages used as experimental data.

\subsubsection{Football}
\begin{figure}[t]
    \centering
    \includegraphics[width=0.45\textwidth]{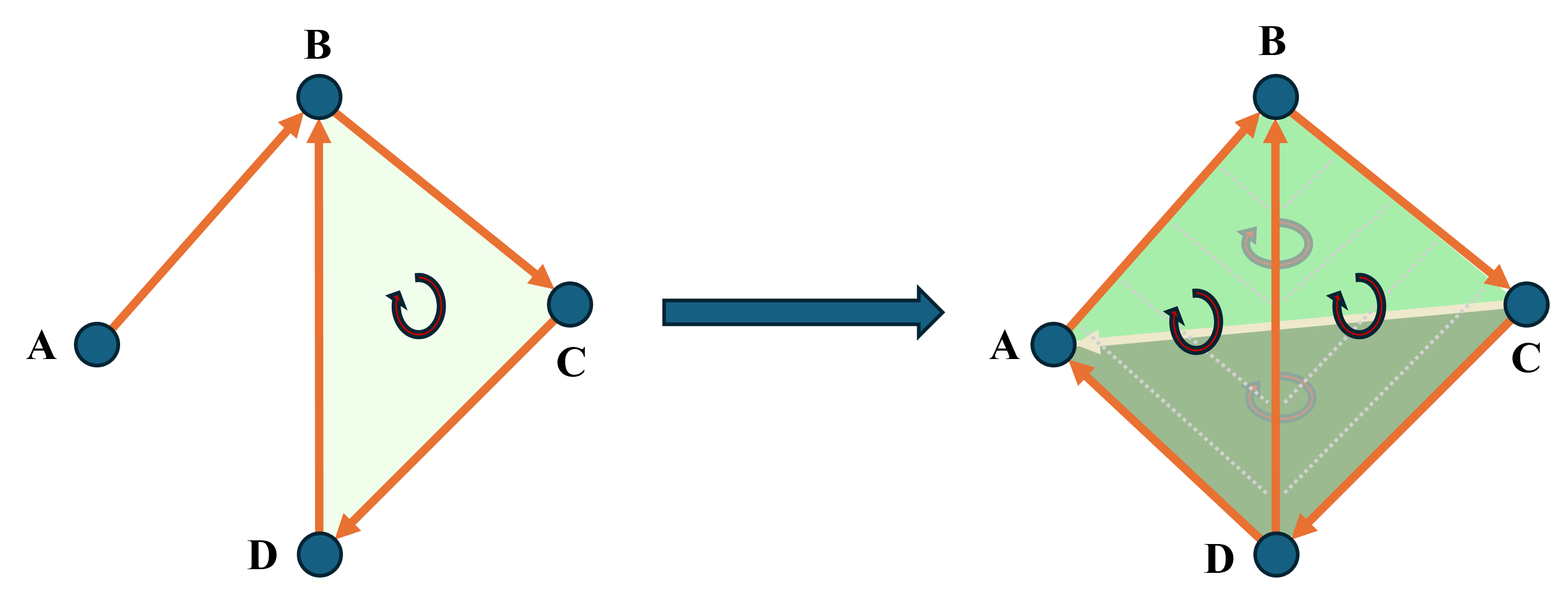}
    \caption{Illustration of the football dataset with four players A, B, C, and D. Suppose the ball is passed along the path A$\rightarrow$B$\rightarrow$C$\rightarrow$D$\rightarrow$B. There is a single passing loop with no passing error. This scenario can be modeled by a simplicial complex with four nodes, six edges, and four triangles. The edge signal is 1 on the edges $\{A,B\}$, $\{B,C\}$, $\{C,D\}$, and $\{D,B\}$, and 0 on the remaining edges. Notice that, except for the start node A and the end node B, all other nodes have zero divergence. Since no passing error occurs, all node signals are zero. Consequently, there is one passing loop $B \rightarrow C \rightarrow D \rightarrow B$, and only the triangle $\{B,C,D\}$ carries a value of 1, while the other triangles have zero signal.}
    \label{fig:Football}
\end{figure}

This dataset considers the passing data from the German national team, collected during their match against England in the 2020 European Championship. Each player is represented as a node. An edge exists between any two players, and a triangle represents the passing loop among three players. The method for acquiring simplicial complex signals is as follows: the node signal corresponds to the total number of passing errors each player made throughout the game; the edge signal reflects the number of passes between two players; and the triangle signal indicates the number of passing loops among three players. When players $A$, $B$, and $C$ form a passing loop, we add a unit to the triangle signal formed by those three players. This construction ensures that the edge signals are approximately divergence-free if no passing error occurs. Since a player receiving the ball passes it on without holding it, only the first and last players have non-zero divergence, as illustrated in Fig.~\ref{fig:Football}.

\subsection{Hodge Subspace Detector}\label{sub:experi_hodge}

\begin{table}[t]
\centering
\caption{Experimental setup for the Hodge subspace detection case.}
\label{tab:exp_setup_hodge}
\begin{tabular}{l|c|c|c}
   \toprule
   \textbf{Dataset} & {\textbf{$\mathcal{H}_0$}} & {\textbf{$\mathcal{H}_1$}} & \textbf{SNR}  \\
   \cmidrule(lr){1-1} \cmidrule(lr){2-2} \cmidrule(lr){3-3} \cmidrule(lr){4-4}  
   \textbf{{Forex}}~\cite{jia2019graph} & Curl-free flow & $\boldsymbol{s}^1 = \mathbf{B}_2 \bar{\boldsymbol{s}}^2$, $\bar{\boldsymbol{s}}^2 \sim \mathcal{N}\left(\mathbf{0},  \mathbf{I}\right)$ &  -10dB  \\
   \textbf{{Lastfm}}~\cite{jia2019graph} & Div-free flow & $\boldsymbol{s}^1 = \mathbf{B}_1^{\top} \bar{\boldsymbol{s}}^0$, $\bar{\boldsymbol{s}}^0 \sim \mathcal{N}\left(\mathbf{0},  \mathbf{I}\right)$ &  -10dB  \\
   \textbf{{Cherry}}~\cite{krishnan2023simplicial} & Curl-free flow &  Flow with curl component&  20dB  \\
   \textbf{{Football}} & Div-free flow & Non-div-free flow &  0dB  \\
   \bottomrule
\end{tabular}
\end{table}

\begin{table}[t]
\centering
\caption{{Area under the curves (AUC) for the complete data. -Th. and -Exp. represent the theoretical and empirical results}}
\begin{tabular}{l|cccc}
   \toprule
   \textbf{{Method}} & \textbf{Forex} & \textbf{Lastfm} & \textbf{Cherry} & \textbf{Football} \\
   \midrule
   HSD-Th.   & 0.80  & 1.00  & 0.82  & 0.73  \\
   HSD-Exp.  & 0.80$\pm$0.01  & 1.00$\pm$0.00  & 0.82$\pm$0.01  & 0.73$\pm$0.01  \\
   DSD-Th.   & 0.99  & 1.00  & 1.00  & 0.95  \\
   DSD-Exp.  & \textbf{0.99$\pm$0.00}  & \textbf{1.00$\pm$0.00}  & \textbf{1.00$\pm$0.00}  & \textbf{0.95$\pm$0.01}  \\
   B-SMSD~\cite{isufi2018blind} & 0.57$\pm$0.02  & 0.67$\pm$0.02  & 0.76$\pm$0.18  & 0.53$\pm$0.01  \\
   \bottomrule
\end{tabular}
\label{tab:AUC_complete}
\end{table}



\begin{figure}[t]
    \centering
    \begin{minipage}[t]{0.45\linewidth}
        \centering
        \includegraphics[width=1.08\linewidth]{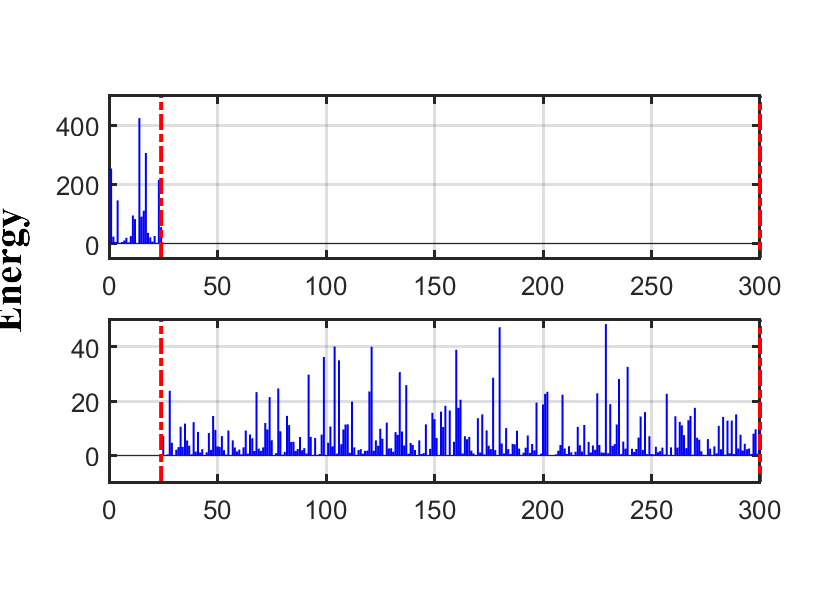}
        \centerline{(a) Forex}
    \end{minipage}%
    \begin{minipage}[t]{0.45\linewidth}
        \centering
        \includegraphics[width=1.08\linewidth]{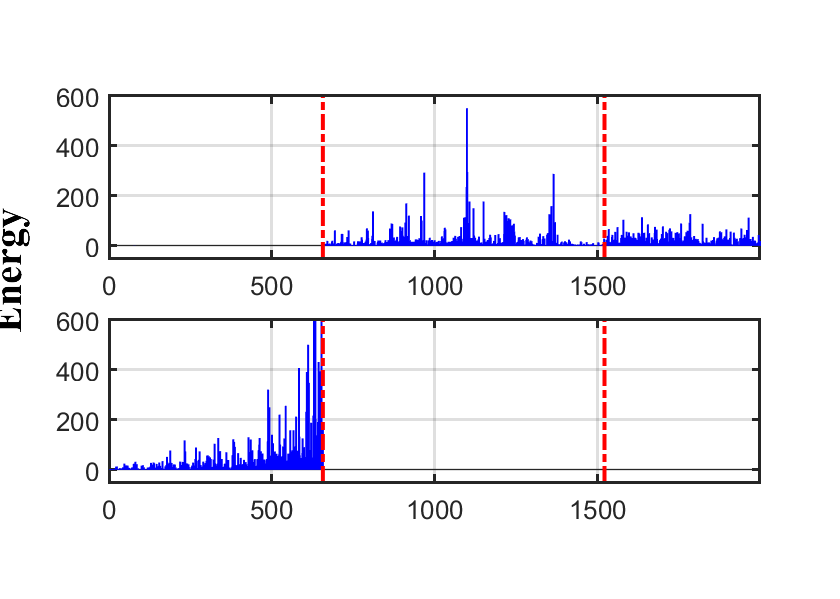}
        \centerline{(b) Lastfm}
    \end{minipage}
    \vfill
    \begin{minipage}[t]{0.45\linewidth}
        \centering
        \includegraphics[width=1.08\linewidth]{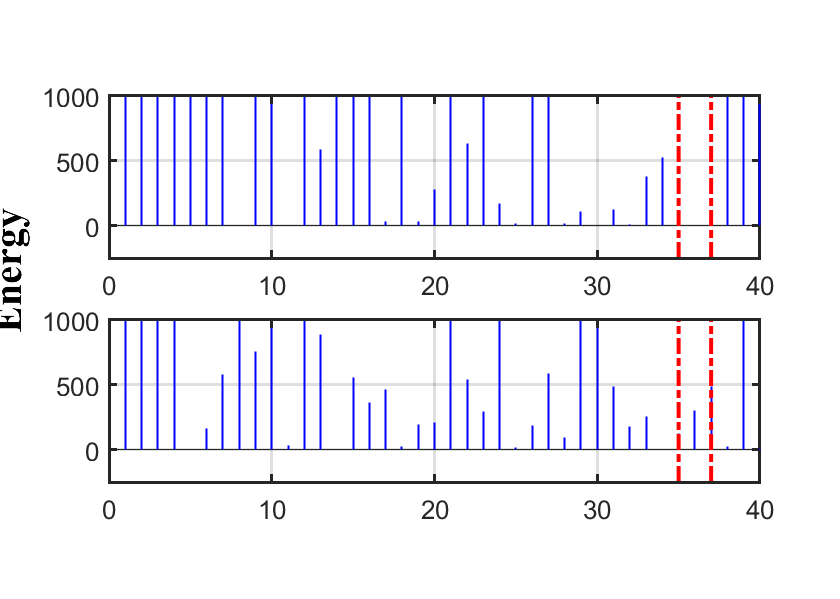}
        \centerline{(c) Cherry}
    \end{minipage}%
    \begin{minipage}[t]{0.45\linewidth}
        \centering
        \includegraphics[width=1.08\linewidth]{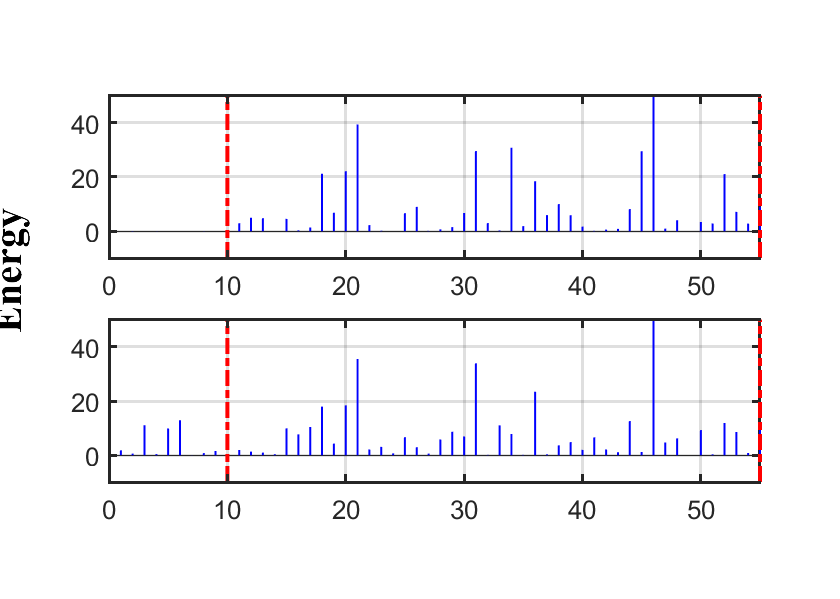}
        \centerline{(d) Football}
    \end{minipage}
    \caption{Energy of projection onto the Hodge subspace for different datasets. In (a), (b), (c) and (d), the upper and lower subgraphs are the energy of the edge signals projection under hypothesis $\mathcal{H}_0$ and $\mathcal{H}_1$ in the Hodge subspaces, respectively. The regions divided by the red lines represent, from left to right, the Hodge gradient, curl and harmonic subspaces.}
    \label{fig:energy_proj}
\end{figure}

\smallskip
\noindent\textbf{Experimental setup.} The experimental setup for the HSD is summarized in Table~\ref{tab:exp_setup_hodge}, while the signal energy projection onto the Hodge subspaces is shown in Fig.~\ref{fig:energy_proj}. Here, we focus on detecting edge signals without considering node and triangle signals.

\emph{Forex:} under hypothesis \(\mathcal{H}_0\), the edge signal represents a foreign exchange rate flow that is curl-free. Under hypothesis \(\mathcal{H}_1\), we generate the flow as \(\boldsymbol{s}^1 = \mathbf{B}_2 \bar{\boldsymbol{s}}^2\), where \(\bar{\boldsymbol{s}}^2 \sim \mathcal{N}\bigl(\mathbf{0}, \mathbf{I}_{N_2}\bigr)\), placing it in the curl subspace. This setup reflects a scenario in which the arbitrage-free condition is violated. We then add zero-mean Gaussian noise at an SNR of -10 dB. The goal is to detect whether the foreign exchange rate satisfies the arbitrage-free condition so that the edge signal is curl-free.

\emph{Lastfm:} under hypothesis \(\mathcal{H}_0\), the edge signal captures a user transition flow, which is divergence-free. Under \(\mathcal{H}_1\), we synthetically generate flows as \(\boldsymbol{s}^1 = \mathbf{B}_1^{\top} \bar{\boldsymbol{s}}^0\), where \(\bar{\boldsymbol{s}}^0 \sim \mathcal{N}\bigl(\mathbf{0}, \mathbf{I}_{N_0}\bigr)\), placing it in the gradient subspace. We then add zero-mean Gaussian noise at an SNR of -10 dB. The objective is to determine whether the edge signal is divergence-free.

\emph{Cherry hills:} edge signals under hypotheses \(\mathcal{H}_0\) and \(\mathcal{H}_1\) correspond to two distinct demand conditions, referred to as demand-0 and demand-1. By controlling demand-0, we ensure that the edge signals under \(\mathcal{H}_0\) are curl-free. We set the water demands to remain constant. Under hypothesis \(\mathcal{H}_0\), the edge signal denotes the flow rate in the pipes under a specified water demand-0. Under hypothesis \(\mathcal{H}_1\), the edge signal corresponds to the flow rate under a different specified water demand-1. As shown in Fig.~\ref{fig:energy_proj}-(c), the flow rate generated by demand-0 resides in the gradient and harmonic subspaces, exhibiting nearly zero projection energy onto the curl subspace. In contrast, the flow rate produced by demand-1 has a non-zero curl component. We then add zero-mean Gaussian noise at an SNR of 20 dB. The task is to identify the demand pattern of the water flow rate by detecting whether the flow is curl-free.

\emph{Football:} under hypothesis \(\mathcal{H}_0\), we construct the edge signal without incorporating passing errors. For example, if a pass from player A is intercepted and ultimately returned to player B, we treat it as a direct pass from A to B, making the passing flow approximately divergence-free and thus modeling a scenario where passing is uninterrupted. Under hypothesis \(\mathcal{H}_1\), we consider the true passing process, where the existence of passing errors results in an edge signal that is not divergence-free. We then add zero-mean Gaussian noise at an SNR of 0 dB to simulate real-world uncertainties and imperfections in the observation or measurement of the passing process. The objective is to determine whether passing interruptions occur by checking whether the edge signal is divergence-free.

For a fair comparison, we set the same energy level for the edge signals under both hypotheses. Our results are averaged over \(1 \times 10^3\) independent noisy realizations of a single sample. We select the area under the curves (AUCs) of the receiver operating characteristics (\(\text{P}_\text{D}\) vs \(\text{P}_\text{FA}\)) as our evaluation metric. For each dataset, we adjust the SNR to emphasize performance differences.

We compare our method with the blind simple matched subspace detector (B-SMSD)~\cite{isufi2018blind}, originally designed for graph signals. To adapt it for edge signals, we first map edges to nodes by constructing the line-graph~\cite{schaub2018flow}. This detector assumes the observed signal is bandlimited with respect to the graph Fourier transform of the line-graph, and then compares the out-of-band SNR with a threshold \(\gamma\). For the (bandwith of the) B-SMSD, we select 95\% of the line-graph eigenvectors corresponding to the smallest eigenvalues.

\smallskip
\noindent\textbf{Results.} The outcomes of the HSD are presented in Table~\ref{tab:AUC_complete}, where theoretical and experimental results align closely, corroborating the proposed theory. Although the Cherry dataset has the highest SNR (see Table~\ref{tab:exp_setup_hodge}), its detection performance for the HSD is not optimal due to the small dimension \(N_{\overline{\Delta}}\) of the complement subspace, making the energy detector in~\eqref{eq:detector_hodge} more noise-sensitive. In contrast, despite having the same SNR as the Forex dataset, Lastfm achieves better detection performance because its larger complement subspace dimension provides greater noise robustness. The baseline B-SMSD method is not effective at distinguishing between hypotheses because its design principle does not match the higher-order detection task. Despite relatively acceptable performance on the Cherry dataset—likely because its underlying topology closely resembles a path graph (see Figure~\ref{fig:cherry}), which approximates its line-graph—the high standard deviation across multiple runs reveals instability. Moreover, on other datasets, the B-SMSD clearly falls short, indicating that a line-graph approach is unsuitable for this task.

\subsection{Dirac Subspace Detector}\label{sub:experi_dirac}


\begin{table*}[t]
\centering
\caption{Experimental setup for the Dirac subspace detector.}
\label{tab:exp_setup_dirac}
\begin{tabular}{l|c|c|c|c|c|c}
   \toprule
   \textbf{Dataset} &{\textbf{$\mathcal{H}_0-node$}}& {\textbf{$\mathcal{H}_0-edge$}} &{\textbf{$\mathcal{H}_0-triangle$}}&{\textbf{$\mathcal{H}_1-node$}}& {\textbf{$\mathcal{H}_1-edge$}} &{\textbf{$\mathcal{H}_1-triangle$}}  \\
   \cmidrule(lr){1-1} \cmidrule(lr){2-2} \cmidrule(lr){3-3} \cmidrule(lr){4-4}  \cmidrule(lr){5-5}\cmidrule(lr){6-6}\cmidrule(lr){7-7}
   \textbf{{Forex}}  &$ \mathbf{B}_1 \bar{\boldsymbol{s}}^1$, $\bar{\boldsymbol{s}}^1 \sim \mathcal{N}\left(\mathbf{0},  \mathbf{I}\right)$& Curl-free flow &$\mathbf{0}$&$\mathbf{0}$& $ \mathbf{B}_2 \bar{\boldsymbol{s}}^2$, $\bar{\boldsymbol{s}}^2 \sim \mathcal{N}\left(\mathbf{0},  \mathbf{I}\right)$ &$ \mathbf{B}_2^\top \bar{\boldsymbol{s}}^1$, $\bar{\boldsymbol{s}}^1 \sim \mathcal{N}\left(\mathbf{0},  \mathbf{I}\right)$  \\
   \textbf{{Lastfm}}  &$\mathbf{0}$& Div-free flow &$ \mathbf{B}_2^\top \bar{\boldsymbol{s}}^1$, $\bar{\boldsymbol{s}}^1 \sim \mathcal{N}\left(\mathbf{0},  \mathbf{I}\right)$&$ \mathbf{B}_1 \bar{\boldsymbol{s}}^1$, $\bar{\boldsymbol{s}}^1 \sim \mathcal{N}\left(\mathbf{0},  \mathbf{I}\right)$& $ \mathbf{B}_1^{\top} \bar{\boldsymbol{s}}^0$, $\bar{\boldsymbol{s}}^0 \sim \mathcal{N}\left(\mathbf{0},  \mathbf{I}\right)$ &$\mathbf{0}$  \\
   \textbf{{Cherry}}  &Pressure-0& Non-curl flow &Area demand-0&Pressure-1&  Flow with curl component&Area demand-1  \\
   \textbf{{Football}} &Passing errors-0& Div-free flow &Passing loops-0&Passing errors-1& Non-div-free flow &Passing loops-1  \\
   \bottomrule
\end{tabular}
\end{table*}

\begin{figure}[t]
    \centering
    \begin{minipage}[t]{0.45\linewidth}
        \centering
        \includegraphics[width=1.1\linewidth]{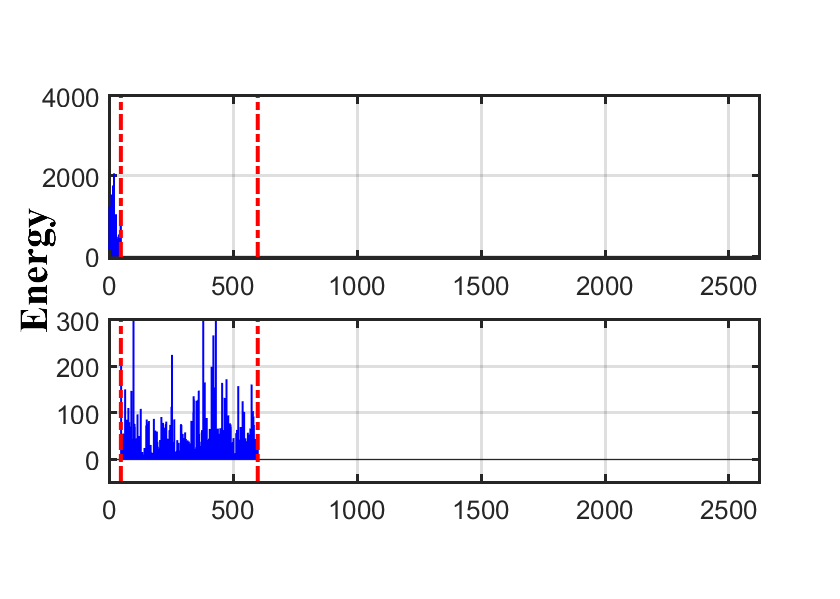}
        \centerline{(a) Forex}
    \end{minipage}%
    \begin{minipage}[t]{0.45\linewidth}
        \centering
        \includegraphics[width=1.1\linewidth]{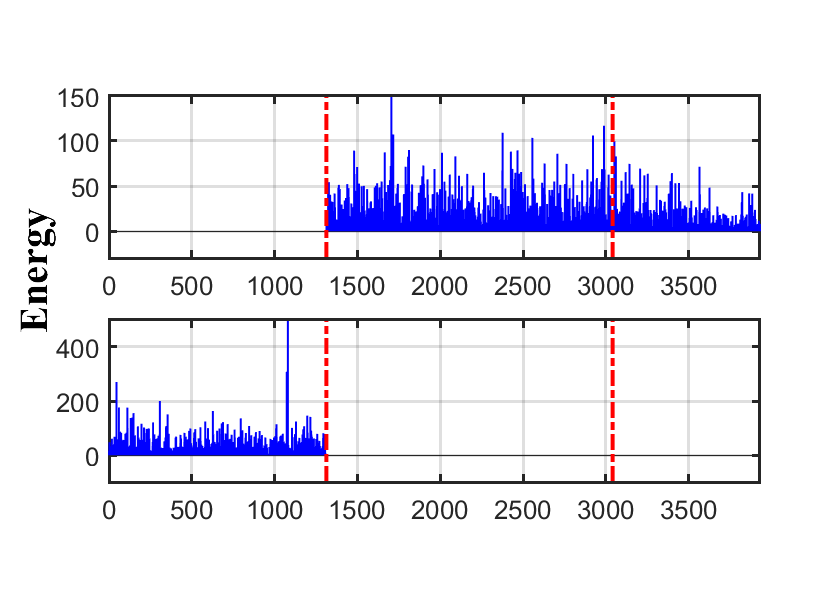}
        \centerline{(b) Lastfm}
    \end{minipage}
    \vfill
    \begin{minipage}[t]{0.45\linewidth}
        \centering
        \includegraphics[width=1.1\linewidth]{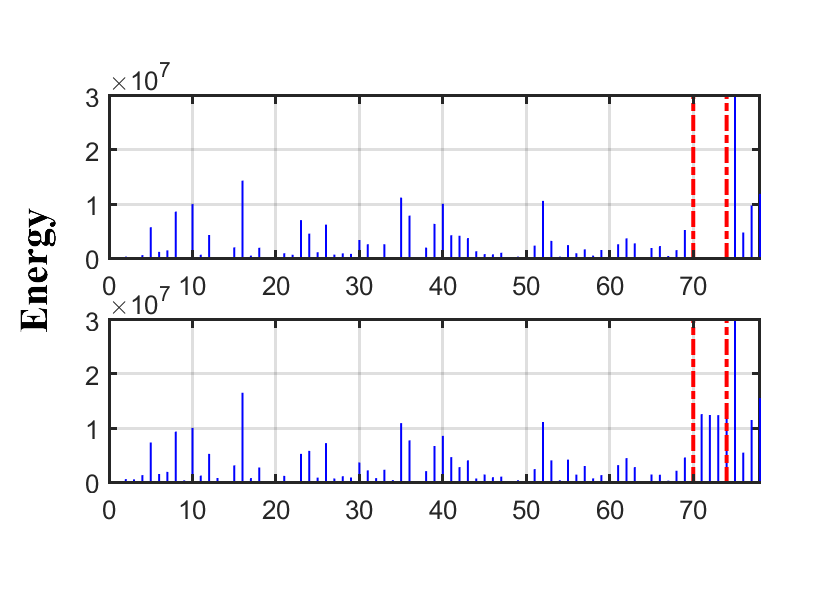}
        \centerline{(c) Cherry}
    \end{minipage}%
    \begin{minipage}[t]{0.45\linewidth}
        \centering
        \includegraphics[width=1.1\linewidth]{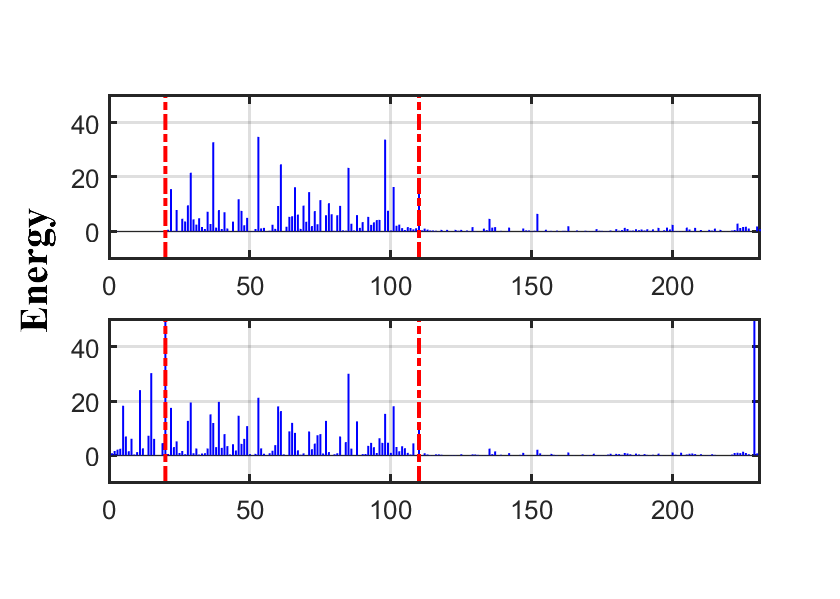}
        \centerline{(d) Football}
    \end{minipage}
    \caption{Energy of the projection onto the Dirac subspace for different datasets. In (a), (b), (c) and (d), the upper and lower subgraphs are the energy of the edge signals projection under hypothesis $\mathcal{H}_0$ and $\mathcal{H}_1$ in the Dirac subspaces, respectively. The regions divided by the red lines represent, in order: the Dirac gradient, curl and harmonic subspace.}
    \label{fig:energy_dirac}
\end{figure}

\smallskip
\noindent\textbf{Experimental setup.} The experimental setup is summarized in Table~\ref{tab:exp_setup_dirac}, and the signal projection energy onto the Dirac subspaces is shown in Fig.~\ref{fig:energy_dirac}. We focus on the detection task that involves node and triangle signals. The edge signals $\boldsymbol{s}^1$ to be detected are the same as those used in the Hodge-based experiments, with the key distinction being the inclusion of node and triangle signals to evaluate their impact on detection performance. Specifically, for the Forex and Lastfm datasets, we generate synthetic node and triangle signals to investigate their influence on the detector’s performance; for the Cherry and Football datasets, we employ real signals. The DSD (zero-padded) in Fig.~\ref{fig:impact} represents a special case in which the node and triangle signals are padded with zeros.

\emph{Forex:} under hypothesis $\mathcal{H}_0$, the node signal is $ \boldsymbol{s}^0 = \mathbf{B}_1 \bar{\boldsymbol{s}}^1$, where $\bar{\boldsymbol{s}}^1 \sim \mathcal{N}\left(\mathbf{0},  \mathbf{I}_{N_1}\right)$, and the triangle signal is $\boldsymbol{s}^2 = \mathbf{0}$. Accordingly, the constructed signal $\boldsymbol{s}$ lies in the Dirac gradient subspace, as illustrated in Fig.~\ref{fig:energy_dirac}-(a). Under hypothesis $\mathcal{H}_1$, the node signal is $\boldsymbol{s}^0 = \mathbf{0}$, and the triangle signal is $ \boldsymbol{s}^2 =\mathbf{B}_2^\top \bar{\boldsymbol{s}}^1$, where $\bar{\boldsymbol{s}}^1 \sim \mathcal{N}\left(\mathbf{0},  \mathbf{I}_{N_1}\right)$. The resulting signal $\boldsymbol{s}$ lies in the Dirac curl subspace.

\emph{Lastfm:} under hypothesis $\mathcal{H}_0$, the node signal is $\boldsymbol{s}^0 = \boldsymbol{0}$, and the triangle signal is $\boldsymbol{s}^2 = \mathbf{B}_2^\top \bar{\boldsymbol{s}}^1$, where $\bar{\boldsymbol{s}}^1 \sim \mathcal{N}\left(\mathbf{0}, \mathbf{I}_{N_1}\right)$. Hence, the constructed signal $\boldsymbol{s}$ occupies the Dirac curl and harmonic subspace without the Dirac gradient component, as shown in Fig.~\ref{fig:energy_dirac}-(b). Under hypothesis $\mathcal{H}_1$, the node signal is $\boldsymbol{s}^0 = \mathbf{B}_1 \bar{\boldsymbol{s}}^1$, where $\bar{\boldsymbol{s}}^1 \sim \mathcal{N}\left(\mathbf{0},  \mathbf{I}_{N_1}\right)$, and the triangle signal is $ \boldsymbol{s}^2 =\mathbf{0}$. As a result, the constructed signal $\boldsymbol{s}$ resides in the Dirac gradient subspace.

\emph{Cherry:} under hypotheses $\mathcal{H}_0$ and $\mathcal{H}_1$, the node signal $\boldsymbol{s}^0$ represents the water pressure at each node, while the triangle signal $\boldsymbol{s}^2$ is the sum of the water demands in each triangular area (i.e., over the three nodes forming a triangle). Consequently, under $\mathcal{H}_0$, the constructed signal $\boldsymbol{s}$ remains in the Dirac gradient and harmonic subspaces, containing no Dirac curl component. Under $\mathcal{H}_1$, however, there is a curl component, as illustrated in Fig.~\ref{fig:energy_dirac}-(c). Here, the priors for $\mathcal{H}_0$ and $\mathcal{H}_1$ stem from observations in specific experimental results.

\emph{Football:} the node and triangle signals capture the passing errors of each player and the passing loops among three players, respectively, as shown in Fig.~\ref{fig:Football}. The resulting signal $\boldsymbol{s}$ spans the Dirac curl and harmonic subspaces. Under $\mathcal{H}_1$, however, $\boldsymbol{s}$ adds a non-zero Dirac gradient component, as depicted in Fig.~\ref{fig:energy_dirac}-(d).

For a fair comparison, we set the energy of the signal $\boldsymbol{s}$ to be equal under both hypotheses. We then add zero-mean Gaussian noise that preserves the same edge-signal SNR used in the Hodge-based experiments.


\begin{figure}[t]
    \centering
    \begin{minipage}[t]{0.5\linewidth}
        \centering
        \includegraphics[width=1.1\linewidth]{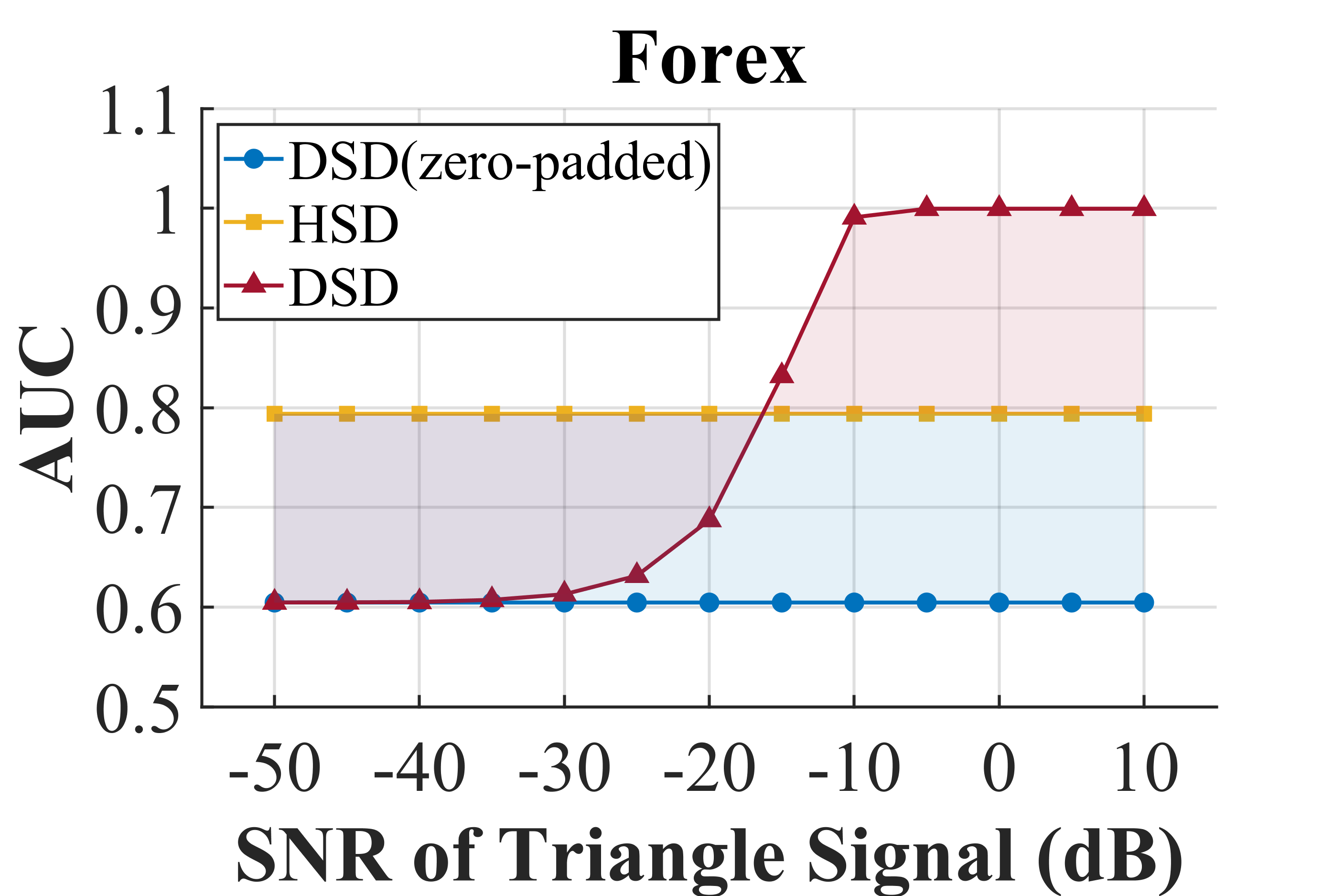}
    \end{minipage}%
    \begin{minipage}[t]{0.5\linewidth}
        \centering
        \includegraphics[width=1.1\linewidth]{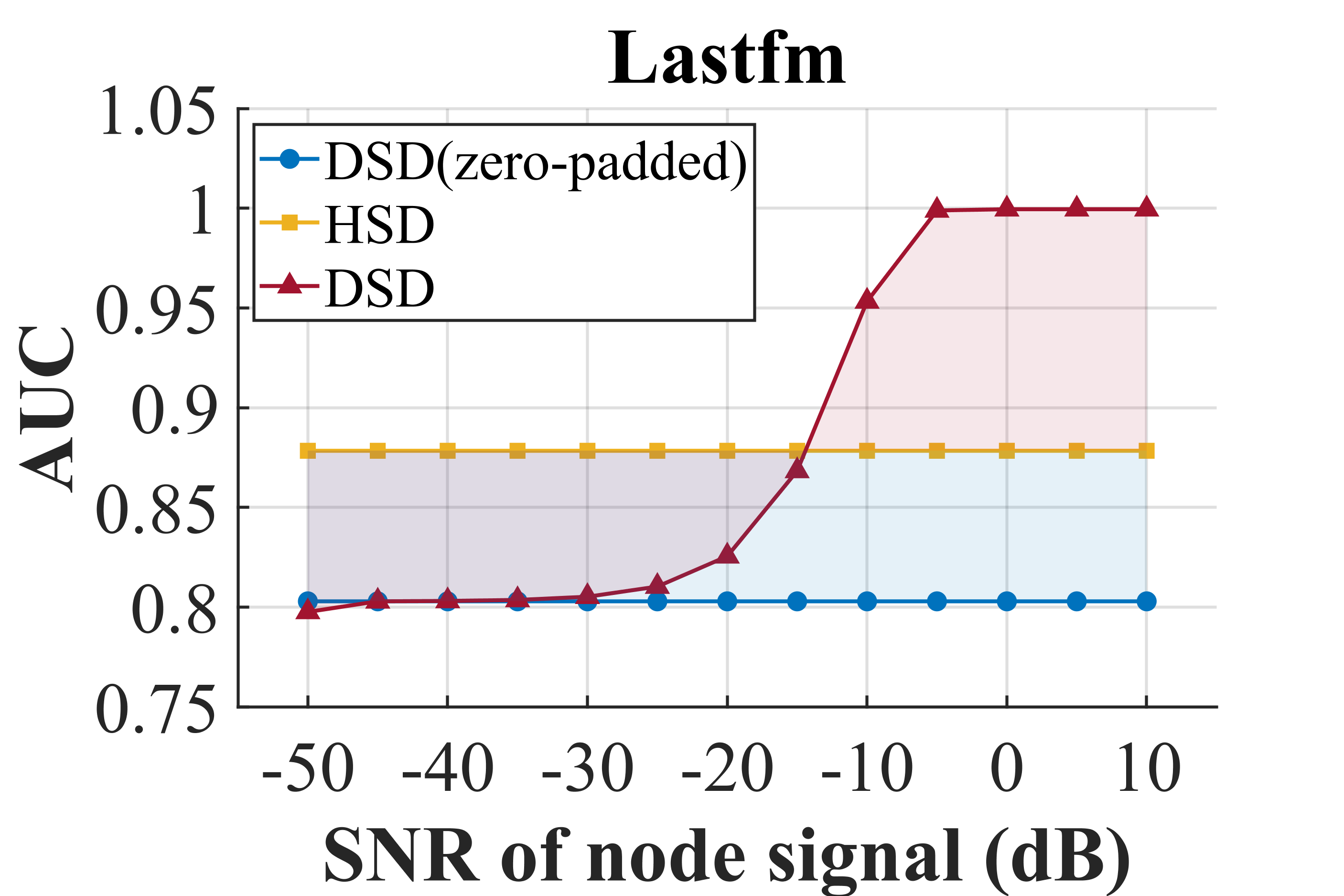}
    \end{minipage}
    \caption{Area under the curves (AUC) for different detectors. The yellow line is the HSD. The blue line is the DSD without considering the information in the node and triangle signals (zero-padded). The red line is the DSD considering the node and triangle signals. The SNRs of the edge signal are -10 and -15 dBs, for the Forex and Lastfm datasets, respectively.}
    \label{fig:impact}
\end{figure}

\smallskip
\noindent\textbf{Results.}
The results for the controlled setting on the Forex and Lastfm datasets are shown in Fig.~\ref{fig:impact}. They indicate that, as the SNR of the node or the triangle signal increases, the detection performance improves gradually. When the energy of the node or triangle signal becomes sufficiently large, the AUC approaches one. The yellow line representing the HSD lies above the blue line for the DSD, which implies that if the node or triangle signal is unknown and therefore zero-padded, the HSD performs better than the DSD. This occurs because the dimension $N_{\ccalP \overline{\Delta}}$ in the Dirac-based experiments is larger than the Hodge-based $N_{\overline{\Delta}}$, while the node or triangle signal does not contribute to the detection when it is zero-padded. Consequently, the deflection coefficient $d^2$ is smaller for the DSD when only the edge signal is considered, resulting in poorer detection performance (c.f. \eqref{eq:pd_energy}). However, as the energy of the node or triangle signals increases, the DSD's performance rises and eventually surpasses that of the HSD. The performance of the DSD is heavily influenced by the properties of the node and triangle signals. Specifically, when these signals reside in certain Dirac subspaces, the Dirac detector can more effectively capture the structural characteristics and outperform the HSD. In practical scenarios, however, node and triangle signals may deviate from these assumptions; under such circumstances, the DSD might no longer surpass the HSD's performance.

From Table~\ref{tab:AUC_complete}, we see that the DSD outperforms every other alternative, including real data node and triangle signals. As noted, adding node and triangle signals enlarges the energy gap in the complement Dirac subspace between hypotheses $\mathcal{H}_0$ and $\mathcal{H}_1$, whereas the HSD considers only the edge signal information.

\subsection{Incomplete data}\label{sub:incomplete}
In this subsection, we address missing data by varying the sampling rate from 0.1 to 1. We compare the performance with that of an interpolation detector. We first interpolate the incomplete data based on prior information, following~\cite{schaub2018flow}, and then perform detection on the interpolated signal. The challenge is that the signals under hypotheses $\mathcal{H}_0$ and $\mathcal{H}_1$ have different priors. Consequently, we leverage the subspace prior of the signal under hypothesis $\mathcal{H}_0$ for the interpolation task, and also under hypothesis $\mathcal{H}_1$ because the exact origin of the noisy signal is unknown. Concretely, we solve problem
\begin{equation}\label{eq:interpolation}
\begin{aligned}
    \argmin_{\hat{\boldsymbol{x}}} \hspace{.2cm} \| \mathbf{Q} \hat{\boldsymbol{x}} \|_2^2  \hspace{.6cm} \text{subject to} \hspace{.2cm} \boldsymbol{\Theta} \hat{\boldsymbol{x}} = \boldsymbol{x},
\end{aligned}
\end{equation}
where the matrix $\mathbf{Q}$ is $\bbU_{\overline{\Delta}}^\top$ or $\mathbf{U}_{\mathcal{P}\overline{\Delta}}^{\top}$ for the Hodge- and Dirac-based experiments, respectively. The matrix $\boldsymbol{\Theta} \in\{0,1\}^{N_o \times N}$ is the sampling matrix, and $\boldsymbol{x}$ is the observation. The objective is to minimize the energy of the interpolated signals in the complement subspaces, given that this energy should be zero under hypothesis $\mathcal{H}_0$. We solve this problem via ADMM.


\begin{figure}[t]
    \centering
    \begin{minipage}[t]{0.45\linewidth}
        \centering
        \includegraphics[width=1.1\linewidth]{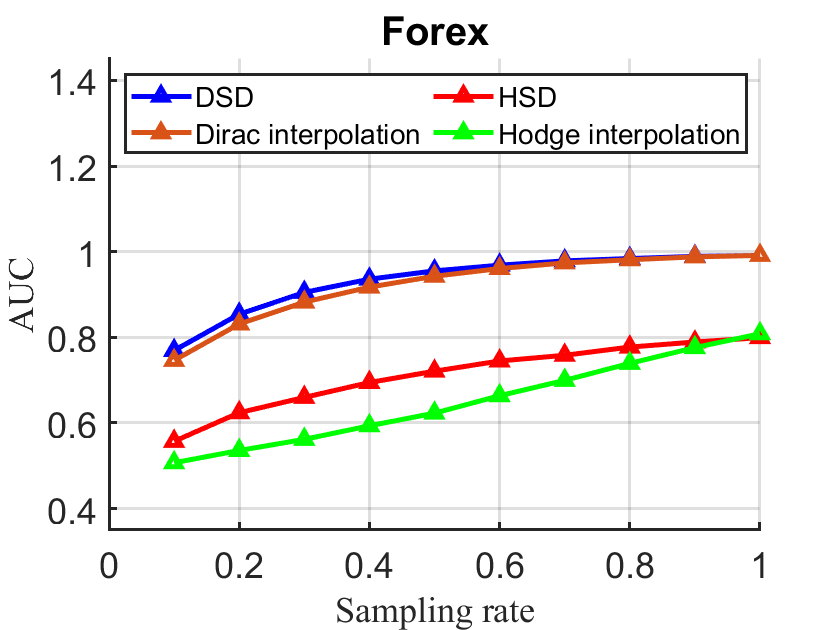}
    \end{minipage}%
    \begin{minipage}[t]{0.45\linewidth}
        \centering
        \includegraphics[width=1.1\linewidth]{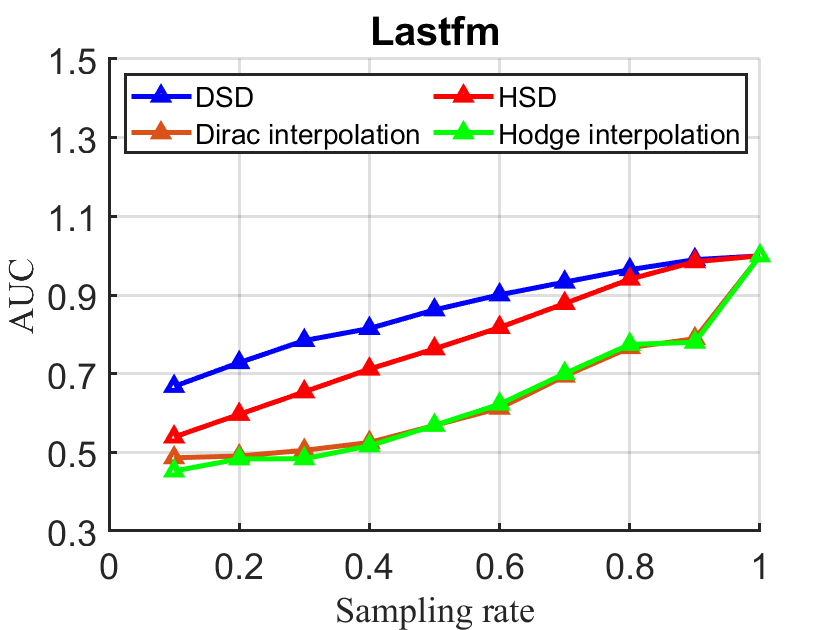}
    \end{minipage}
    \vfill
    \begin{minipage}[t]{0.45\linewidth}
        \centering
        \includegraphics[width=1.1\linewidth]{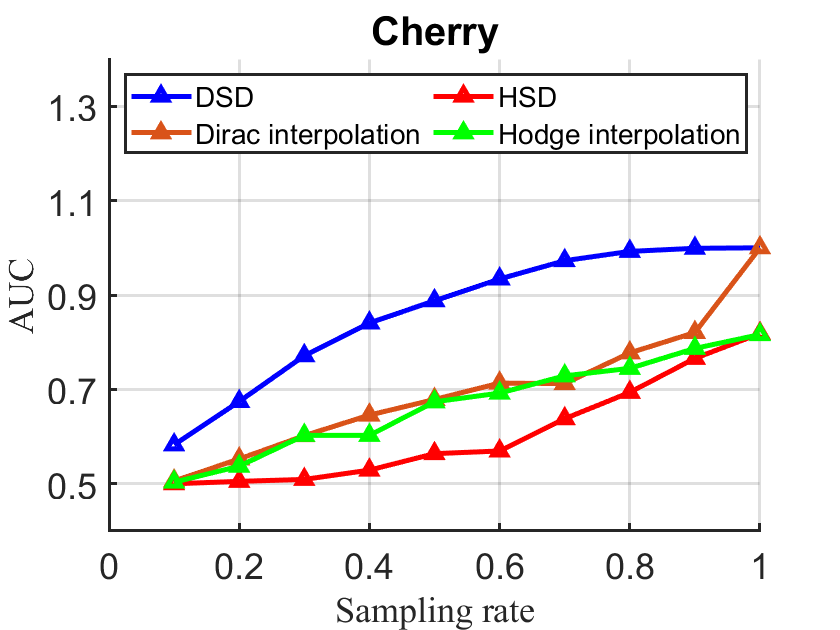}
    \end{minipage}%
    \begin{minipage}[t]{0.45\linewidth}
        \centering
        \includegraphics[width=1.1\linewidth]{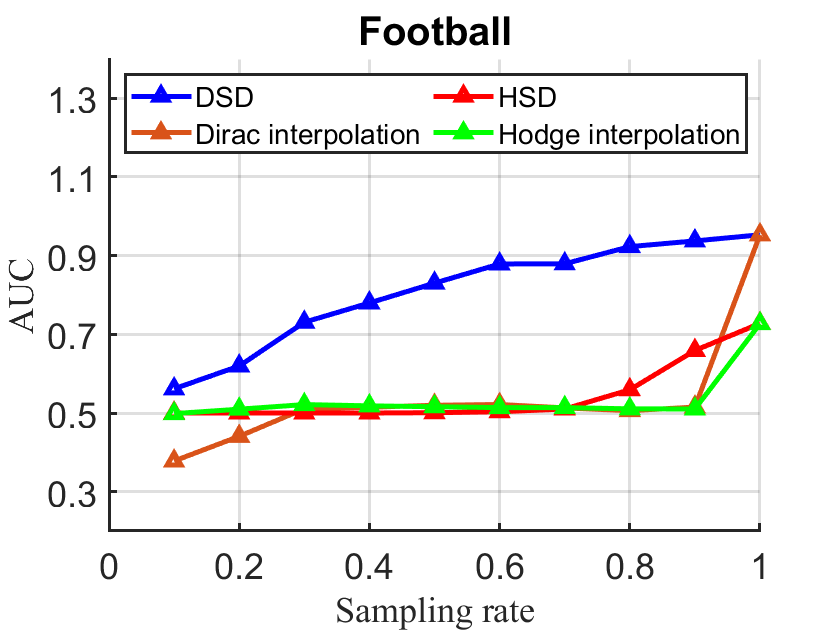}
    \end{minipage}
    \caption{Area under the curves (AUC) for the incomplete data. The percentage of the missing data is ranging from 0\% to 90\%. }
\label{fig:AUC_incomplete}
\end{figure}

\smallskip\noindent
\textbf{Overdetermined case.}
Figure~\ref{fig:AUC_incomplete} presents the results, where the proposed GLRT-based detector proves effective for both the HSD and DSD. The DSD consistently performs better because the information contributed by the node or triangle signals bolsters the detection task, even when some of the data is missing. Across different datasets, the performance of the HSD and DSD varies, largely due to the node or triangle signals' differing power levels. When these signals hold greater energy (e.g., in the Football dataset), the DSD’s improvement over the HSD becomes more pronounced.

The performance of the interpolation detector differs significantly among datasets because it is not an optimal solution, and its effectiveness depends heavily on the specific properties of the underlying signals. The interpolation baseline does not perform well here because the regularizer in~\eqref{eq:interpolation} imposes an inaccurate prior for the signal under $\mathcal{H}_1$, reducing its efficacy. Moreover, because the matrix $\mathbf{Q}$ is fat, the solution to the interpolation problem is non-unique, and ADMM-based outcomes lack stability. This instability is one of the primary causes of the interpolation detector’s poor performance.

\smallskip\noindent
\textbf{Underdetermined case.} Finally, we assess the underdetermined scenario in~\eqref{eq:test_missing}. We set the regularizer in~\eqref{eq:mle_underdet} based on the prior knowledge of $\hat{\boldsymbol{s}}_{j}$. Employing synthetic data on the topologies of these four datasets allows us to verify that incorporating the prior information on the signal can enhance detection performance. For simplicity, we only examine the underdetermined cases in the Dirac setting, as described in Table~\ref{tab:dirac_setup}, where $i$ indexes the vector. We have prior knowledge that both simplicial embeddings $\hat{\boldsymbol{s}}_{0}$ and $\hat{\boldsymbol{s}}_{1}$ are low-pass. Thus, in~\eqref{eq:mle_underdet}, we set $\lambda_j \Omega (\hat{\boldsymbol{s}}_{j})$ as $\lambda_j\|\mathbf{R}_j \hat{\boldsymbol{s}}_{j}\|_2^2$, where $\mathbf{R}_j$ is diagonal with decreasing diagonal entries.

As shown in Fig.~\ref{fig:AUC_overdet}, the underdetermined detector that incorporates prior knowledge of the signal outperforms approaches lacking this information. By contrast, the interpolation detector yields suboptimal performance because it fails to effectively utilize accurate prior knowledge.

\begin{figure}[t]
    \centering
    \begin{minipage}[t]{0.45\linewidth}
        \centering
        \includegraphics[width=1.1\linewidth]{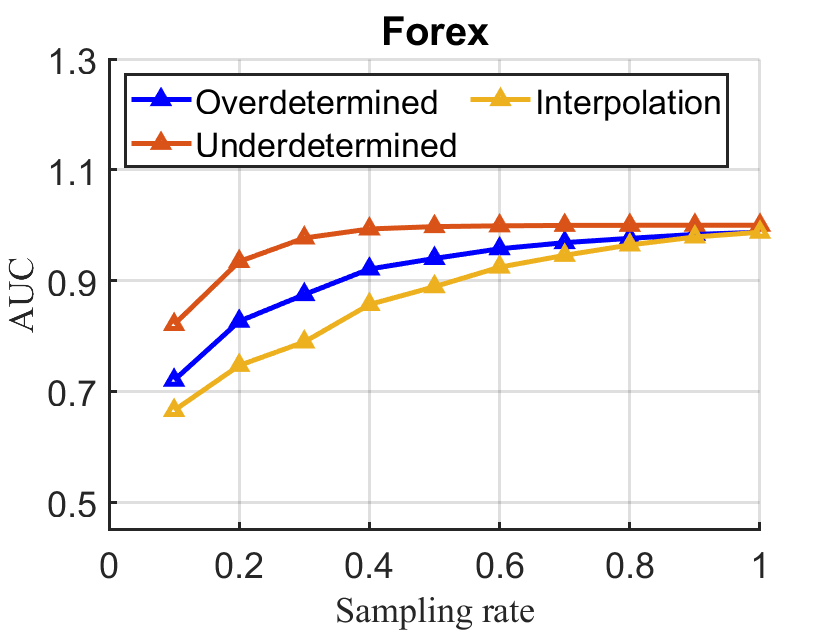}
    \end{minipage}%
    \begin{minipage}[t]{0.45\linewidth}
        \centering
        \includegraphics[width=1.1\linewidth]{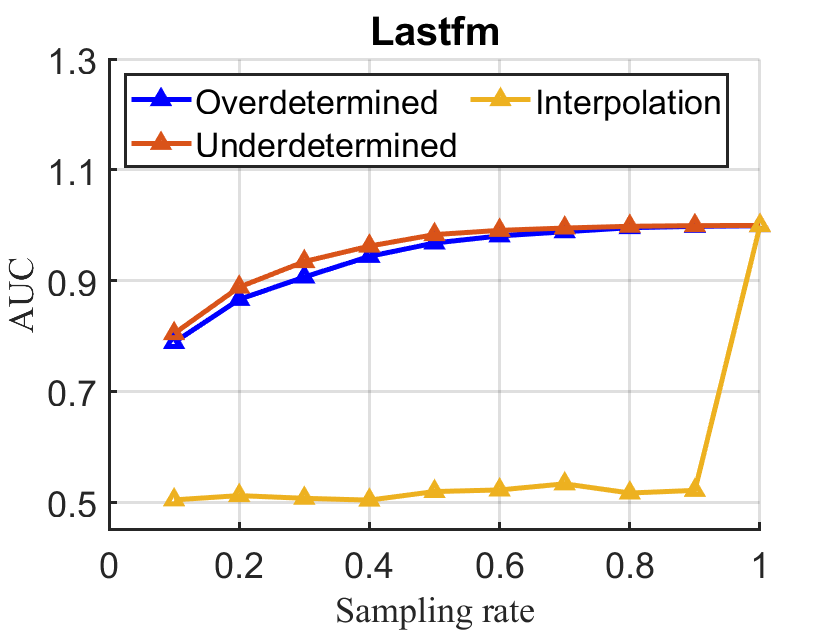}
    \end{minipage}
    \vfill
    \begin{minipage}[t]{0.45\linewidth}
        \centering
        \includegraphics[width=1.1\linewidth]{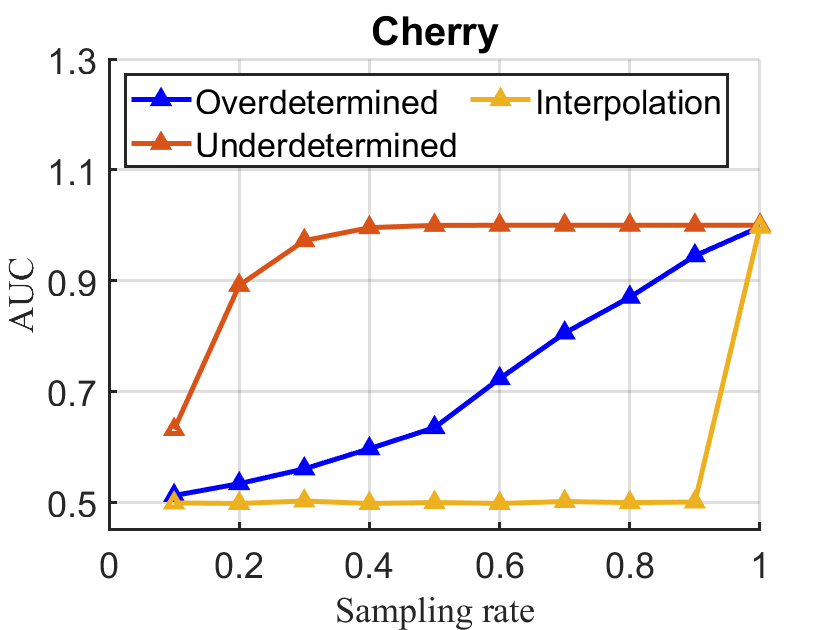}
    \end{minipage}%
    \begin{minipage}[t]{0.45\linewidth}
        \centering
        \includegraphics[width=1.1\linewidth]{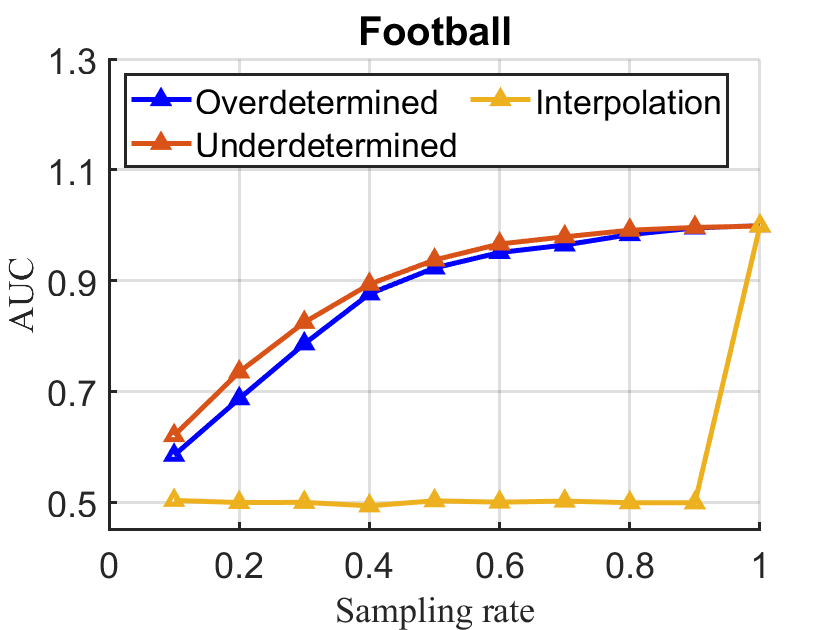}
    \end{minipage}
    \caption{Area under the curves (AUC) for overdetermined and underdetermined cases. We assume the prior information of the $\hat{\boldsymbol{s}}_{0}$ and $\hat{\boldsymbol{s}}_{1}$ are both low-pass.}
\label{fig:AUC_overdet}
\end{figure}

\begin{table*}[t]
\centering
\caption{Experimental setup for the edge signals in the underdetermined experiment. }
\begin{tabular}{l|c|c|c|c|c}
   \toprule
   \textbf{Dataset} & {\textbf{$\mathcal{H}_0$}} & {\textbf{$\mathcal{H}_1$}} &$[\lambda_0\mathbf{R}_0]_{ii}$&$[\lambda_1\mathbf{R}_1]_{ii}$& \textbf{SNR}  \\
   \cmidrule(lr){1-1} \cmidrule(lr){2-2} \cmidrule(lr){3-3} \cmidrule(lr){4-4}  \cmidrule(lr){5-5} \cmidrule(lr){6-6}
   \textbf{Forex}~\cite{jia2019graph} & 
    $[\hat{\boldsymbol{s}}_{0}]_i \sim \mathcal{N}\left(\text{exp}(-i/20)^{\top},  10^{-3}\right)$ & 
    $[\hat{\boldsymbol{s}}_{1}]_i \sim \mathcal{N}\left(\text{exp}(-i/1000)^{\top},  10^{-3}\right)$ & 
    $0.01*\text{exp}(i/50)$ & 
    $\text{exp}(i/2000)$ & 
    -10dB  \\
   \textbf{{Lastfm}}~\cite{jia2019graph} & $[\hat{\boldsymbol{s}}_{0}]_i \sim \mathcal{N}\left(\text{exp}(-i/1000)^{\top},  10^{-3}\right)$ & $[\hat{\boldsymbol{s}}_{1}]_i \sim \mathcal{N}\left(\text{exp}(-i/3000)^{\top},  10^{-3}\right)$ &0.01*$\text{exp}(i/50)$&$\text{exp}(i/2000)$&  -10dB  \\
   \textbf{{Cherry}}~\cite{krishnan2023simplicial} & $[\hat{\boldsymbol{s}}_{0}]_i \sim \mathcal{N}\left(\text{exp}(-i/20)^{\top},  10^{-3}\right)$ & $[\hat{\boldsymbol{s}}_{1}]_i \sim \mathcal{N}\left(\text{exp}(-i/30)^{\top},  10^{-3}\right)$ &0.01*$\text{exp}(i/2)$&$\text{exp}(i/100)$&  20dB  \\
   \textbf{{Football}} & $[\hat{\boldsymbol{s}}_{0}]_i \sim \mathcal{N}\left(\text{exp}(-i/100)^{\top},  10^{-3}\right)$ & $[\hat{\boldsymbol{s}}_{1}]_i \sim \mathcal{N}\left(\text{exp}(-i/200)^{\top},  10^{-3}\right)$ &0.01*$\text{exp}(i/5)$&$\text{exp}(i/50)$&  0dB  \\
   \bottomrule
\end{tabular}
\label{tab:dirac_setup}
\end{table*}

\section{Conclusion} \label{S:conclusion}

This paper proposed an MSD framework to determine whether a simplicial complex signal resides in a specific subspace of interest via hypothesis testing. We first applied the methodology to $k$-signals (node, edge, or triangle signals) to detect membership in gradient, curl, harmonic, or combined subspaces of the Hodge Laplacian. We then extended our approach to simplicial complex signals using the Dirac operator and its associated subspaces, thereby establishing a theoretical link between the Hodge and Dirac frameworks. The resulting detector, which is optimal under a GLRT perspective, leverages the signal's energy in the orthogonal complement of the target subspace. Recognizing the prevalence of missing data in real-world signals, we also developed an optimal detector for incomplete observations.

We evaluated our proposed MSD on four real-world simplicial complexes, two of which include real simplicial signals residing in Dirac subspaces. The results demonstrated superior performance by (i) considering the entire simplicial signal and (ii) employing the GLRT-optimal detector.

Future work will focus on extending this framework to other topological spaces—such as cell complexes or hypergraphs—and on exploring the task of jointly detecting and localizing anomalies in the simplicial subspaces.


\appendix

\subsection{Proof of Proposition~\ref{prop:asymptotic}} \label{sub:proof_prop_asymptotic}

Adapting the derivations provided in~\cite{kay1998fundamentals}, we have the following proof. To determine the asymptotic performance of an energy detector, we need to solve for its first second-order moments. The Chi-square distribution with the non-centraility parameter $\left\|\hat{\boldsymbol{s}}_{\overline{\Delta}}\right\|_2^2/\sigma^2$ and $N_{\overline{\Delta}}$ degrees of freedom :
\begin{equation}
\left\{\begin{aligned}
& E\left(T(\hat{\boldsymbol{x}}_{\overline{\Delta}}) ; \mathcal{H}_0\right)=N_{\overline{\Delta}} \\
& E\left(T(\hat{\boldsymbol{x}}_{\overline{\Delta}}) ; \mathcal{H}_1\right)=\left\|\hat{\boldsymbol{s}}_{\overline{\Delta}}\right\|_2^2/\sigma^2+N_{\overline{\Delta}} \\
& \operatorname{var}\left(T(\hat{\boldsymbol{x}}_{\overline{\Delta}}) ; \mathcal{H}_0\right)=2 N_{\overline{\Delta}}\\
& \operatorname{var}\left(T(\hat{\boldsymbol{x}}_{\overline{\Delta}}) ; \mathcal{H}_1\right)=4 \left\|\hat{\boldsymbol{s}}_{\overline{\Delta}}\right\|_2^2/\sigma^2+2 N_{\overline{\Delta}}
\end{aligned}\right..
\end{equation}
Thus, the false alarm and the detection probability of the energy detector can be expressed respectively as
\begin{equation}
\text{P}_{\mathrm{FA}}=Q\left(\frac{\gamma-N_{\overline{\Delta}}}{\sqrt{2 N_{\overline{\Delta}}}}\right)
\end{equation}
\begin{equation}
\text{P}_{\mathrm{D}}=Q\left(\frac{\gamma -\left\|\hat{\boldsymbol{s}}_{\overline{\Delta}}\right\|_2^2/\sigma^2 - N_{\overline{\Delta}}}{\sqrt{4 \left\|\hat{\boldsymbol{s}}_{\overline{\Delta}}\right\|_2^2/\sigma^2+2N_{\overline{\Delta}}}}\right)
\end{equation}
Following a standard routine, the detection probability can be written into a function of the false alarm probability as
\begin{equation}
\text{P}_{\mathrm{D}}=Q\left(\frac{Q^{-1}\left(\text{P}_{\mathrm{FA}}\right)-\sqrt{\frac{N_{\overline{\Delta}}}{2}} \frac{\left\|\hat{\boldsymbol{s}}_{\overline{\Delta}}\right\|_2^2/\sigma^2}{N_{\overline{\Delta}}}}{\sqrt{1+2 \frac{\left\|\hat{\boldsymbol{s}}_{\overline{\Delta}}\right\|_2^2/\sigma^2}{N_{\overline{\Delta}}}}}\right)
\end{equation}
When the number of degrees of freedom $N_{\overline{\Delta}}$ is large, the term $\left\|\hat{\boldsymbol{s}}_{\overline{\Delta}}\right\|_2^2/(\sigma^2 N_{\overline{\Delta}}) \approx 0$. Hence, by expanding the argument of the $Q$ function using a first-order Taylor expansion,  the detection probability is approximated as
\begin{equation}
\text{P}_{\mathrm{D}} \approx Q\left(Q^{-1}\left(\text{P}_{\mathrm{FA}}\right)-\sqrt{\frac{(\left\|\hat{\boldsymbol{s}}_{\overline{\Delta}}\right\|_2^2/\sigma^2)^2}{2N_{\overline{\Delta}}}}\right),
\end{equation}
which concludes the proof.

\subsection{Proof of Proposition \ref{prop:hodge_dirac}}
\label{sub:proof_prop_nonzeropad}

We start with~\eqref{eq:dirac_to_hodge1}.
The condition $[\boldsymbol{s}^0\|\boldsymbol{s}^1\| \boldsymbol{s}^2]\in \operatorname{span}\left(\mathbf{D}_l\right)$ indicates that there exists a nonzero simplicial signal $ [\widetilde{\boldsymbol{s}}^0\|\widetilde{\boldsymbol{s}}^1\| \widetilde{\boldsymbol{s}}^2]$ that satisfies
\begin{equation}
\left[\begin{array}{c}
\boldsymbol{s}^0 \\
\boldsymbol{s}^1\\
\boldsymbol{s}^2 
\end{array}\right] = \left[\begin{array}{ccc}
\mathbf{0} & \mathbf{B}_1 & \mathbf{0} \\
\mathbf{B}_1^{\top} & \mathbf{0} & \mathbf{0} \\
\mathbf{0} & \mathbf{0} & \mathbf{0}
\end{array}\right]\left[\begin{array}{c}
\widetilde{\boldsymbol{s}}^0 \\
\widetilde{\boldsymbol{s}}^1\\
\widetilde{\boldsymbol{s}}^2
\end{array}\right]=\left[\begin{array}{c}
\mathbf{B}_1\widetilde{\boldsymbol{s}}^1 \\
\mathbf{B}_1^\top\widetilde{\boldsymbol{s}}^0\\
\mathbf{0} 
\end{array}\right].
\end{equation}
This means that $\boldsymbol{s}^0 = \mathbf{B}_1\widetilde{\boldsymbol{s}}^1$, $\boldsymbol{s}^1 = \mathbf{B}_1^\top\widetilde{\boldsymbol{s}}^0\in \operatorname{span}\left(\mathbf{B}_1^\top\right)$ and $\boldsymbol{s}^2 = \mathbf{0}$, which proves~\eqref{eq:dirac_to_hodge1} completes. 

The proof of~\eqref{eq:dirac_to_hodge2} is analogous to that of~\eqref{eq:dirac_to_hodge1}.

To prove~\eqref{eq:dirac_to_hodge3}, we note that the condition $\left[\boldsymbol{s}^0\|\boldsymbol{s}^1\| \boldsymbol{s}^2\right]\in \operatorname{kernel}\left(\mathbf{D}\right)$ implies that $\mathbf{D} \left[\boldsymbol{s}^0\|\boldsymbol{s}^1\| \boldsymbol{s}^2\right] = \boldsymbol{0}$. Multiplying both sides of this equality by $\mathbf{D}$ yields $\mathbf{D}^2 \left[\boldsymbol{s}^0\|\boldsymbol{s}^1\| \boldsymbol{s}^2\right] = \mathbf{D} \boldsymbol{0} = \boldsymbol{0}$. Next, use the definition of $\mathbf{D}^2$ and write
\begin{equation}
 \left[\begin{array}{ccc}
\mathbf{L}_0 & \mathbf{0} & \mathbf{0} \\
\mathbf{0} & \mathbf{L}_1 & \mathbf{0} \\
\mathbf{0} & \mathbf{0} & \mathbf{L}_2
\end{array}\right]\left[\begin{array}{c}
\boldsymbol{s}^0 \\
\boldsymbol{s}^1\\
\boldsymbol{s}^2 
\end{array}\right]=\left[\begin{array}{c}
\mathbf{L}_0\boldsymbol{s}^0 \\
\mathbf{L}_1\boldsymbol{s}^1\\
\mathbf{L}_2\boldsymbol{s}^2 
\end{array}\right]=\mathbf{0}. 
\end{equation}
This implies that $\mathbf{L}_1\boldsymbol{s}^1 = \mathbf{0}\Rightarrow\boldsymbol{s}^1\in \operatorname{kernel}\left(\mathbf{L}_1\right)$, completing the proof of~\eqref{eq:dirac_to_hodge3} and, as a result, the proof of Proposition~\ref{prop:hodge_dirac}.

\bibliographystyle{IEEEtran}
\bibliography{mybib}

\end{document}